\documentclass[11pt]{article}

\usepackage[dvipsnames]{xcolor}
\usepackage[T1]{fontenc}
\usepackage{tikz}
\usetikzlibrary{arrows.meta}
\usepackage{array}
\usepackage[utf8]{inputenc}
\usepackage[english]{babel}
\usepackage{times}
\usepackage{pgfplots}
\usepgfplotslibrary{groupplots}
\usepackage{graphicx,wrapfig,lipsum}
\usepackage{float}    %
\usepackage{verbatim} %
\usepackage{amsmath, amssymb, amsthm, mathtools}  %
\usepackage{caption}
\usepackage{subcaption}%
\usepackage[colorlinks,linkcolor=blue, citecolor=cyan, urlcolor=magenta, anchorcolor=ForestGreen,bookmarks=False]{hyperref}
\usepackage{cleveref}
\usepackage{bookmark}
\usepackage{fullpage}
\usepackage{enumerate}
\usepackage{paralist}
\usepackage{xspace}
\usepackage{bbm}
\usepackage{bm}
\usepackage{url}
\usepackage{lipsum}
\usepackage{algorithm}
\usepackage{algpseudocode} 
\usepackage{color}
\usepackage{microtype}
\usepackage[section]{placeins}
\usepackage{lscape}
\usepackage{multirow}
\usepackage[textsize=small]{todonotes}
\newcommand{\vertiii}[1]{{\left\vert\kern-0.25ex\left\vert\kern-0.25ex\left\vert #1
		\right\vert\kern-0.25ex\right\vert\kern-0.25ex\right\vert}}

\makeatletter

\newcommand{\vect}[1]{\ensuremath{\mathbf{#1}}}
\newcommand{\x}{\vect{x}}

\newcommand{\mSigma}{\bm{\Sigma}}

\newcommand{\A}{\vect{A}}

\newcommand{\PAR}{\mathrm{PAR}}

\newcommand{\ridge}{\mathrm{ridge}}
\newcommand{\Lasso}{\mathrm{Lasso}}

\newcommand{\prox}{\mathbf{prox}}

\DeclareMathOperator{\clip}{clip}

\DeclareMathOperator{\supp}{supp}

\newcommand{\mylinewidthone}{0.6pt}
\newcommand{\mylinewidthtwo}{1.2pt}

\newcommand{\inner}[2]{\left\langle #1, #2 \right\rangle}

\newcommand{\norm}[1]{\left\lVert#1\right\rVert}

\newcommand{\tr}{{\operatorname{tr}}}

\newcommand{\qr}{\operatorname{qr}}

{\medskip}

{\hspace*{\fill}$\Box$\par\vspace{4mm}}
{\hspace*{\fill}$\Box$\par}

\newcommand{\cA}{\mathcal{A}}

\newcommand{\cC}{\mathcal{C}}
\newcommand{\cD}{\mathcal{D}}
\newcommand{\cE}{\mathcal{E}}

\newcommand{\cN}{\mathcal{N}}

\newcommand{\cQ}{\mathcal{Q}}
\newcommand{\cR}{\mathcal{R}}

\newcommand{\bE}{\mathbb{E}}

\newcommand{\bP}{\mathbb{P}}

\newcommand{\bR}{\mathbb{R}}

\newcommand{\bZ}{\mathbb{Z}}

\crefname{assumption}{assumption}{assumptions}
\crefname{condition}{condition}{conditions}

\newtheorem{theorem}{Theorem}

\newtheorem{corollary}{Corollary}
\newtheorem{proposition}{Proposition}
\newtheorem{example}{Example}

\newtheorem{assumption}{Assumption}

\usepackage{amsmath,amsfonts,bm}

\DeclareMathOperator{\col}{col}
\newcommand*{\rom}[1]{%
\textup{\uppercase\expandafter{\romannumeral#1}}%
}

\def\eqref#1{equation~\ref{#1}}

\def\1{\bm{1}}

\def\vzero{{\bm{0}}}

\def\vepsilon{{\bm{\epsilon}}}

\def\va{{\bm{a}}}
\def\vb{{\bm{b}}}

\def\vv{{\bm{v}}}

\def\vx{{\bm{x}}}
\def\vy{{\bm{y}}}
\def\vz{{\bm{z}}}

\def\mA{{\bm{A}}}
\def\mB{{\bm{B}}}
\def\mC{{\bm{C}}}

\def\mI{{\bm{I}}}
\def\mJ{{\bm{J}}}

\def\mSigma{{\bm{\Sigma}}}

\DeclareMathAlphabet{\mathsfit}{\encodingdefault}{\sfdefault}{m}{sl}
\SetMathAlphabet{\mathsfit}{bold}{\encodingdefault}{\sfdefault}{bx}{n}

\DeclareMathOperator*{\argmin}{arg\,min}

\DeclareMathOperator{\sign}{sign}

\def\vdelta{{\bm{\delta}}}

	\title{Quantization through Piecewise-Affine Regularization:\\ Optimization and Statistical Guarantees}

\author{Jianhao Ma\\
University of Pennsylvania
\\
\texttt{jianhaom@wharton.upenn.edu}\and Lin Xiao\\
Meta FAIR\\
\texttt{linx@meta.com}}
	
\begin{document}
\maketitle

\begin{abstract}
Optimization problems over discrete or quantized variables are very challenging in general due to the combinatorial nature of their search space. 
Piecewise-affine regularization (PAR) provides a flexible modeling and computational framework for quantization based on continuous optimization. In this work, we focus on the setting of supervised learning and investigate the theoretical foundations of PAR from optimization and statistical perspectives. 
First, we show that in the overparameterized regime, where the number of parameters exceeds the number of samples, every critical point of the PAR-regularized loss function exhibits a high degree of quantization.
Second, we derive closed-form proximal mappings for various (convex, quasi-convex, and non-convex) PARs and show how to solve PAR-regularized problems using the proximal gradient method, its accelerated variant, and the Alternating Direction Method of Multipliers.
Third, we study statistical guarantees of PAR-regularized linear regression problems; specifically, we can approximate classical formulations of $\ell_1$-, squared $\ell_2$-, and nonconvex regularizations using PAR and obtain similar statistical guarantees with quantized solutions. 
\end{abstract}

\section{Introduction}
In many machine learning and decision-making problems, we need to optimize an objective function where some variables are constrained to be discrete:
\begin{equation}
\min_{\vx\in \cQ^{d_1}, \vy\in \bR^{d_2}} f(\vx, \vy). \label{eq::direct-quantization}
\end{equation}
Here $\vy$ is the continuous variable, and the elements of $\vx$  are restricted to a discrete set $\cQ$. For example, $\cQ$ can be the set of binary values $\{0, 1\}$, a subset of integers, or a finite set of discrete real numbers.
The prevalence and importance of such problems are highlighted by the following examples. 
\begin{itemize}
    \item \textbf{Quantization in machine learning model compression.} Modern deep learning models offer remarkable capabilities in vision and language processing, but they often come with substantial computational and memory requirements. Quantization, which maps model parameters from high-precision to low-precision formats, has emerged as an effective approach for model compression. It can significantly reduce memory footprint, computational cost, and inference latency \cite{han2016compression,sze2017efficient}.
    \item \textbf{Communications and signal processing.} Quantization plays a fundamental role in digital communications, where continuous-amplitude signals must be converted into discrete values for transmission, storage, and processing \cite{proakis2008digital,oppenheim1999discrete,gray2002quantization}. This process underlies analog-to-digital conversion, enabling real-world analog signals to be represented with finite bit-depth. The quality of the quantization directly affects signal fidelity, bandwidth efficiency, and error rates in communication systems.
    \item \textbf{Mixed integer programming in operations research.} Many optimization problems in operations research require variables to take discrete values, such as facility location, production scheduling, or resource allocation \cite{snyder2019fundamentals}. These problems are formulated as mixed integer programs where some variables are constrained to integer values while others remain continuous, leading to challenging computational problems that combine combinatorial and continuous optimization \cite{wolsey1999integer,wolsey2020integer}.
\end{itemize}
Solving Problem~(\ref{eq::direct-quantization}) exactly is extremely challenging due to its combinatorial nature. For instance, even the simple convex quadratic binary program $\min_{\vx\in\{0,1\}^d}\vx^{\top}\mA\vx$ with $\mA\succeq 0$ is NP-hard \cite{hartmanis1982computers}.

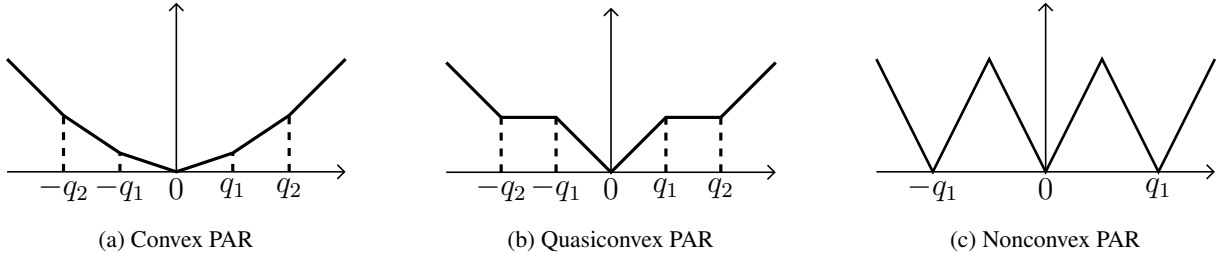
\begin{figure}
    \centering
     \begin{subfigure}[b]{0.3\textwidth}
         \centering
         \begin{tikzpicture}[scale=0.75]
    \tikzset{myarrow/.style={-{Straight Barb[scale=1]}, line width=\mylinewidthone}}
    \tikzset{thickline/.style={line width=\mylinewidthtwo}}
    
    \draw[myarrow] (-3,0) -- (3,0);
    \draw[myarrow] (0,0) -- (0,3);

    \draw[thickline]
        (-3,2) -- (-2,1) -- (-1,0.333) -- (0,0)
        -- (1,0.333) -- (2,1) -- (3,2);

    \draw[dashed, line width=1.2pt] (-2,0) -- (-2,1);
    \draw[dashed, line width=1.2pt] (-1,0) -- (-1,0.333);
    \draw[dashed, line width=1.2pt] (2,0) -- (2,1);
    \draw[dashed, line width=1.2pt] (1,0) -- (1,0.333);
    
    \node at (-2,-0.3) {\large $-q_2$};
    \node at (-1,-0.3) {\large $-q_1$};
    \node at (0,-0.3) {\large $0$};
    \node at (1,-0.3) {\large $q_1$};
    \node at (2,-0.3) {\large $q_2$};

\end{tikzpicture}
         \caption{Convex PAR}
         \label{fig:convex}
     \end{subfigure}
     \hfill
     \begin{subfigure}[b]{0.3\textwidth}
         \centering
         \begin{tikzpicture}[scale=0.73]
    \tikzset{myarrow/.style={-{Straight Barb[scale=1]}, line width=\mylinewidthone}}
    \tikzset{thickline/.style={line width=\mylinewidthtwo}}
    
    \draw[myarrow] (-3,0) -- (3,0);
    \draw[myarrow] (0,0) -- (0,3);

    \draw[thickline]
        (-3,2) -- (-2,1) -- (-1,1) -- (0,0)
        -- (1,1) -- (2,1) -- (3,2);

    \draw[dashed, line width=1.2pt] (-2,0) -- (-2,1);
    \draw[dashed, line width=1.2pt] (-1,0) -- (-1,1);
    \draw[dashed, line width=1.2pt] (2,0) -- (2,1);
    \draw[dashed, line width=1.2pt] (1,0) -- (1,1);
    
    \node at (-2,-0.3) {\large $-q_2$};
    \node at (-1,-0.3) {\large $-q_1$};
    \node at (0,-0.3) {\large $0$};
    \node at (1,-0.3) {\large $q_1$};
    \node at (2,-0.3) {\large $q_2$};

\end{tikzpicture}
          \caption{Quasiconvex PAR}
         \label{fig:quasiconvex}
     \end{subfigure}
     \hfill
    \begin{subfigure}[b]{0.3\textwidth}
         \centering
         \begin{tikzpicture}[scale=0.75]
    \tikzset{myarrow/.style={-{Straight Barb[scale=1]}, line width=\mylinewidthone}}
    \tikzset{thickline/.style={line width=1.0pt}}
    
    \draw[myarrow] (-3,0) -- (3,0);
    \draw[myarrow] (0,0) -- (0,3);

    \draw[thickline]
        (-3,2) -- (-2,0) -- (-1,2) -- (0,0)
        -- (1,2) -- (2,0) -- (3,2);

    \node at (-2,-0.3) {\large $-q_1$};
    \node at (0,-0.3) {\large $0$};
    \node at (2,-0.3) {\large $q_1$};

\end{tikzpicture}
          \caption{Nonconvex PAR}
         \label{fig:nonconvex}
     \end{subfigure}
     \caption{Three different types of PAR $\Psi(\cdot)$ for inducing quantization towards $\cQ=\{\pm q_i\}$.}
    \label{fig::three-different-pars}
\end{figure}

In this paper, we focus on finding \emph{approximate solutions} using continuous optimization methods, by adding regularizations/penalties that can induce discrete solutions to the objective function.
Specifically, we consider the following unconstrained optimization problem
\begin{equation}
    \min_{\vx\in \bR^d} F_{\lambda}(\vx):= f(\vx)+\lambda \Psi(\vx),
    \label{eq::regularized-loss}
\end{equation}
where $\Psi(\cdot)$ is a regularization that encourages the variables to be discrete (within the set~$\cQ$), and $\lambda$ controls the strength of regularization.\footnote{Here we omit the continuous part $\vy$ and focus on the case where all the variables are subject to regularization. It can be easily extended to the general setting by setting $\Psi(\vy)=0$.}
Among many possible choices, the family of Piecewise-Affine Regularizers (PARs) are particularly suited for inducing quantization due to their nature of nonsmoothness (\Cref{fig::three-different-pars} illustrates three representative types of PAR).
In particular, they tend to trap the optimization variables at the set of nondifferentiable breakpoints. 
This mirrors how the $\ell_1$-regularizer promotes sparsity through nondifferentiability at zero. For PARs, we align their set of nondifferentiable points with the target quantization values, making them effective for inducing desired quantization. We refer to the Problem~(\ref{eq::regularized-loss}) with such a regularizer as \textbf{Piecewise-Affine Regularized Optimization (PARO)}.

PAR has a long history in statistics. The classic example is the Lasso \cite{tibshirani1996regression,ChenDonohoSaunders1998}, which uses $\ell_1$-regularization to induce sparsity in linear regression and other statistical learning tasks. Subsequent work introduced other convex PARs to induce different structures; a prominent case is the graph-based total-variation penalty, which suppresses jumps and favors piecewise-constant signals \cite{condat2017discrete,kolmogorov2016total}. Several recent studies examined the geometry of solutions produced by convex PARs \cite{everink2024geometry,schneider2022geometry,tardivel2021geometry}. 

\begin{sloppypar}
    Parallel advances in machine learning highlight an intrinsic link between PARs and quantization. In particular, \cite{courbariaux2015binaryconnect} introduced the straight-through estimator (STE) which has become a workhorse for quantization-aware training, and \cite{yin2018binaryrelax} draws its connection to the framework of regularized optimization. \cite{bai2019proxquant} reinterpret STE as regularized dual averaging \cite{xiao2010rda} with a non-convex PAR (illustrated in \Cref{fig:nonconvex}), 
    and follow-up works broaden this analysis to wider families of PARs, devising optimization algorithms with rigorous convergence guarantees \cite{dockhorn2021demystifying,lu2024understanding}. Most recently, \cite{jin2025parq} propose to use convex PARs (\Cref{fig:convex}) for quantization, and introduce an aggregated proximal stochastic gradient method with last-iterate convergence guarantee and strong empirical results on deep learning tasks.
\end{sloppypar}

While PARs have achieved strong empirical success, their theoretical properties remain underexplored. In particular, the mechanisms by which PARs promote discrete structure, their optimization guarantee, and statistical behavior are not well understood. In this work, we bridge this gap by analyzing PARs through the lenses of quantization, optimization, and statistics, establishing new theoretical foundations that explain and support their empirical effectiveness. Our main contributions are:
\begin{itemize}
    \item \textbf{Quantization guarantees.} 
    We provide theoretical backing for PAR's ability to induce high quantization rate in supervised learning models. We show that all critical points of PARO have a high proportion of quantized entries. The quantization rate is directly linked to the ratio of parameter dimension to sample size, indicating that overparameterized models naturally achieve higher quantization rate.
    \item \textbf{Optimization methods.} We derive closed-form expressions for the proximal mappings of various PARs, spanning convex, quasiconvex, and nonconvex formulations. We show that the proximal gradient method and its accelerated variants can efficiently converge to critical points of PARO for both convex and nonconvex cases. Additionally, for structured problems like linear regression, we demonstrate that the Alternating Direction Method of Multipliers (ADMM) \cite{boyd2011distributed} can be effectively applied, often achieving faster convergence.
    \item \textbf{Statistical properties.} We demonstrate that PARs can closely approximate a wide range of conventional regularizers, including squared $\ell_2$-, $\ell_1$-, and general nonconvex regularizers. For linear regression, we prove that specific PARs can effectively mimic classic regularizers, achieving optimal statistical guarantees with quantized solutions. This highlights PARO's ability to reduce model size without sacrificing performance.
    \item \textbf{Numerical experiments.} We conduct extensive simulations across linear and logistic regression tasks. These experiments empirically corroborate our theoretical findings, validating the quantization, optimization, and statistical guarantees provided by the PARO framework.
\end{itemize}

The rest of this paper is organized as follows. \Cref{sec::PAR} introduces different types of PAR along with their first-order optimality conditions. We also provide quantization guarantees of PAR for generalized linear models and supervised learning. In \Cref{sec::optimization}, we derive the proximal operators of various PARs and illustrate how to solve PARO using several standard optimization algorithms. In \Cref{sec::statistical-guarantee}, we establish statistical guarantees for various PARs in the linear regression setting. Finally, in \Cref{sec::numerical-experiments}, we present numerical experiments that support our theoretical findings and demonstrate the effectiveness of our approach.

\paragraph{Notations.}  
We use bold lowercase letters (e.g., $\vx, \vy, \vz$) to denote vectors, and bold uppercase letters (e.g., $\mA, \mB, \mC$) to denote matrices. For a vector $\vx \in \bR^d$, we define its $\ell_2$-norm as $\norm{\vx} = (\sum_{i=1}^{d} x_i^2)^{1/2}$ and its $\ell_\infty$-norm as $\norm{\vx}_\infty = \max_{i} |x_i|$. The sign function is denoted as $\sign(x)$, which returns $1$ if $x > 0$, $-1$ if $x < 0$, and $0$ otherwise. For a set $I$, we denote its cardinality by $|I|$. Given a matrix $\mA \in \bR^{m \times d}$, we use $\A_i$ to denote its $i$-th row. The notation $\mA_I$ refers to the submatrix of $\mA$ containing only the rows indexed by $I$. For a function $f: \bR^d \to \bR$, we denote its gradient at $\vx$ by $\nabla f(\vx)$ if $f$ is differentiable at $\vx$ and its Clarke subdifferential by $\partial f(\vx)$ otherwise. For a scalar $x$, we use $\lfloor x\rceil$ to denote the closest integer to $x$.

\section{PAR and quantization guarantees}
\label{sec::PAR}

In this paper, we consider coordinate-wise piecewise-affine regularization (PAR), i.e., 
\[
\Psi(\vx)=\sum_{i=1}^d\Psi(x_i),
\]
where we use $\Psi$ for both vector and scalar inputs (slight abuse of notation).
For ease of presentation and broad applicability in practice, we focus on PARs that are symmetric with respect to the origin (as illustrated in \Cref{fig::three-different-pars}), with the definition
\begin{equation}\label{eqn:par-def}
\Psi(x) = a_k(|x| - q_k) + b_k \quad \text{if } q_k \leq |x| \leq q_{k+1},
\end{equation}
where $\mathcal{Q} = \{0, \pm q_1, \ldots, \pm q_m\}$, with $0 = q_0 < q_1 < \cdots < q_m$, denotes the set of targeted quantization values.
The slopes $\cA=\{a_0, a_1, \cdots, a_m\}$ satisfy $a_k \neq a_{k+1}$ for all $0 \leq k \leq m-1$, and the intercepts $\{b_0, b_1, \cdots, b_m\}$ are determined recursively by setting $b_0 = 0$ and
\begin{equation}
b_k = b_{k-1} + a_{k-1}(q_k - q_{k-1}), \qquad k = 1, \ldots, m.
\end{equation}

The remainder of this section presents general optimality conditions for PARO, followed by quantization guarantees established under a supervised learning framework.

\subsection{Optimality conditions}
In this section, we examine the first-order optimality conditions for PARO.\footnote{The results for convex PARs were previously presented in \cite{jin2025parq}; here, we extend them to general PARs and include the full statements for completeness.} %
Specifically, we assume that~$f$ is differentiable and $\vx^\star$ is a Clarke critical point \cite[Definition~2]{bolte2007clarke} of PARO (\eqref{eq::regularized-loss}), i.e., 
\begin{equation}
    \bm{0}\in\nabla f(\vx^\star)+\lambda\partial\Psi(\vx^\star),
\end{equation}
where $\partial\Psi(\vx^\star)$ is the Clarke subdifferential of $\Psi$ at $\vx^\star$.
This can be rewritten as
\(
\nabla f(\vx^\star) \in -\lambda\,\partial \Psi(\vx^\star),
\)
which yields the following conditions for each coordinate $i=1,\cdots,d$ and quantization level $k=1,\cdots,m$:
\begin{alignat}{4}
\nonumber
x_i^\star & = -q_k, & & \Longleftarrow & \nabla_i f(\vx^\star) & \in \lambda\, (a_{k-1}, a_k)\\ %
\nonumber
x_i^\star & \in \!(-q_k, -q_{k-1}) ~& & \Longrightarrow ~\quad & \nabla_i f(\vx^\star) & = \lambda\,a_{k-1}\\ %
\nonumber
x_i^\star & = 0 & & \Longleftarrow & \!\!\!\!\!\! -\nabla_i f(\vx^\star) & \in \lambda\, (-a_0, a_0) \\ %
\nonumber
x_i^\star & \in (q_{k-1},q_k)\quad & & \Longrightarrow ~\quad & \nabla_i f(\vx^\star) & =-\lambda\,a_{k-1} \\ %
\nonumber
x_i^\star & = q_k, & & \Longleftarrow &  \nabla_i f(\vx^\star) & \in \lambda\, (-a_k, -a_{k-1}). %
\end{alignat}
Here the symbol $\Longleftarrow$ ($\Longrightarrow$) means that the left-hand side expression is a necessary (sufficient) condition for the right-hand side expression.

We immediately recognize that the sufficient condition for $x_i^\star=0$ is the same as for the $\ell_1$-regularization $\Psi(x)=\lambda a_0\cdot |x|$. 
Further examination reveals that 
for any parameter not clustered at a discrete value in~$\mathcal{Q}$, i.e., if $x_i^\star\in(q_{k-1},q_k)$, the corresponding partial derivative of~$f$ must equal to one of the $2m$ discrete values $\{\pm\lambda a_{k-1}\}_{k=1}^m$. 
Conversely, almost all values of the partial derivatives of~$f$, except for these $2m$ discrete values,
can be balanced by assigning the model parameters at the $2m+1$ discrete values in~$\mathcal{Q}$.
Intuitively, this implies that the model parameters at optimality are more likely to be clustered at these discrete values. 

\subsection{Quantization guarantee}
\label{sec::quantization-guarantee}
In this section, we study the quantization properties of solutions obtained using the PARO framework
\begin{equation}
    \min_{\vx\in \bR^d} F_{\lambda}(\vx)= f(\vx)+\lambda \Psi(\vx).
\end{equation}
We start with a simple generalized linear model setting, and later generalize it to general supervised learning setting. Given a specific PAR $\Psi(\cdot)$ with quantization values $\cQ=\{0, \pm q_1, \ldots, \pm q_m\}$, for an arbitrary point $\vx\in \bR^d$, we define its \textbf{quantization rate} $\qr(\vx)$ as the fraction of coordinates with quantized values, i.e.,
\begin{equation}
    \qr(\vx):=\frac{|\{i: x_i\in \cQ\}|}{d}.
\end{equation}

\paragraph{Generalized linear model.} We consider objective functions of the form $f(\vx)=g(\mA\vx)$. Here $g(\cdot)$ is a smooth function, and $\mA\in \bR^{n\times d}$ is a matrix representing the input data. This formulation includes various generalized linear models, such as linear regression and logistic regression. The following result provides a quantization guarantee for the critical points of this problem when regularized by a class of PARs.

\begin{theorem}[Quantization guarantee for generalized linear models]
\label{thm::main}
    Consider a PAR $\Psi(\cdot)$ where the slopes $\cA$ satisfies $0\notin \cA$. Suppose each row of the design matrix $\mA\in \bR^{n\times d}$ with $n\leq d$ is i.i.d. drawn from a distribution $\cD_\va$ that is absolutely continuous with respect to the $p$-dimensional Lebesgue measure. Then, with probability one, any critical point $\vx^{\star}$ of PARO (\eqref{eq::regularized-loss}) satisfies $\qr(\vx^{\star})\geq 1-n/d$.
\end{theorem}
Below, we highlight the implications and significance of Theorem~\ref{thm::main}:
\begin{itemize}
    \item \textbf{Effect of overparameterization.} \Cref{thm::main} establishes a lower bound on the quantization rate, namely $1 - \frac{n}{d}$. Thus, in highly overparameterized models where $d \gg n$, the quantization rate approaches $1$. This suggests that larger models are inherently easier to quantize, which is consistent with empirical observations reported in \cite{chen2025scaling,cao2024scaling}.
    \item \textbf{Guarantees for critical points.} This result applies to all critical points of the regularized objective, not just global minimizers. This is particularly useful, as we show in \Cref{sec::optimization} that standard optimization algorithms such as the proximal gradient method can efficiently find such critical points.
    \item \textbf{Tightness of the quantization guarantee.} The lower bound on the quantization rate depends only on the ratio of sample size to data dimension and is independent of the regularization strength. While this may appear weak, extensive simulations in \Cref{sec::numerical-experiments} demonstrate that this bound is nearly tight, particularly when the regularization strength is moderate.
\end{itemize}

\paragraph{General supervised learning.}
We consider a training dataset \(S_n=\{(\va_i, b_i)\}_{i=1}^{n}\) consisting of \(n\) data points, where \(\va_i\in \bR^p\) represents the input and \(b_i\in \bR\) represents the output. We assume that each data point \((\va_i, b_i)\) is i.i.d. generated from an underlying distribution \(\mathcal{D}=\cD_\va\times \cD_b\). We consider a family of machine learning models $f(\va; \vx)$ indexed by parameter $\vx\in \bR^d$. Then, we optimize the PAR-regularized empirical risk function
\begin{equation}
    \min_{\vx\in \bR^d} F_{\lambda}(\vx):=\frac{1}{n}\sum_{i=1}^{n}\ell\left(f(\va_i; \vx), b_i\right)+\lambda \Psi(\vx). \label{eq::ERM}
\end{equation}
Here $\Psi(\cdot)$ is a PAR with quantization values $\cQ$ and the slopes $\cA$.

We now introduce a set of conditions on the data distribution, the model, and the PAR.

\begin{assumption}
\label{assumption::quantization-guarantee}
    The marginal data distribution $\cD_\va$, the model $f(\va; \vx)$, and the PAR regularizer $\Psi(\cdot)$ satisfy the following conditions:
    \begin{itemize}
        \item \textbf{(Continuous data distribution)} The marginal data distribution $\cD_\va$ is absolutely continuous with respect to the $p$-dimensional Lebesgue measure. 
        \item \textbf{(Real analytic model)} The model $f:\bR^p\times \bR^d\to \bR$ is a real analytic function.
        \item \textbf{(Non-degenerate Jacobian)} For any critical point $\vx^\star$ and any index set $I$ with $|I| = n + 1$, and for any vectors $\vv \in \bR^d$, we all have
        \begin{equation}
            \det\left(\left[\nabla_{\vx} f(\va_1; \vx^\star), \cdots, \nabla_{\vx} f(\va_n; \vx^\star), \vv\right]_I\right)\not\equiv 0.
        \end{equation}
        \item \textbf{(Non-degenerate PAR)} For the PAR $\Psi(\cdot)$, $0$ is not in its slope set $\cA$.
    \end{itemize}
\end{assumption}
These conditions are relatively mild. First, we only assume absolute continuity of the data distribution and impose no further requirements such as sub-Gaussianity or bounded moments. Second, the real analyticity of the model $f$ is a standard assumption in the machine learning theory literature, and is satisfied by many common architectures, including feed-forward neural networks with analytic activation functions such as sigmoid or tanh; see, e.g., \cite{schmidt2023taylor, nguyen2017loss, watanabe2010asymptotic, kawaguchi2021recipe}. Third, the non-degenerate Jacobian condition rules out trivial models (e.g., constant functions) and ensures sufficient expressiveness. A similar assumption is also used in \cite{kawaguchi2021recipe}. Lastly, the non-degenerate PAR condition is necessary: if $0 \in \mathcal{A}$, the PAR family includes the zero function $\Psi(x) \equiv 0$, which yields no quantization effect.

\begin{theorem}[Quantization guarantee]
\label{thm::main-extended}
    Suppose \Cref{assumption::quantization-guarantee} holds. Then, with probability one, the quantization rate of any critical point $\vx^{\star}$ of the objective in \Cref{eq::ERM} is at least $1 - n/d$, i.e.,
\begin{equation}
    \qr(\vx^{\star}) \geq 1 - \frac{n}{d}.
\end{equation}
\end{theorem}

Finally, we prove \Cref{thm::main,thm::main-extended}. We begin with the proof of \Cref{thm::main-extended}, and then verify that the generalized linear model setting satisfies \Cref{assumption::quantization-guarantee}. This allows us to apply \Cref{thm::main-extended} and thereby establish \Cref{thm::main} as a special case.
\begin{proof}[Proof of \Cref{thm::main-extended}]
    For any critical point $\vx^\star$ of the objective in \Cref{eq::ERM}, it satisfies the first-order optimality condition 
    \begin{equation}
        \vzero\in \mJ\vdelta+\lambda \partial \Psi(\vx^\star).
    \end{equation}
    Here $\mJ$ is the Jacobian matrix defined by 
    \begin{equation}
        \mJ=\left[\nabla_{\vx} f(\va_1; \vx^\star), \cdots, \nabla_{\vx} f(\va_n; \vx^\star)\right]\in \bR^{d\times n},
    \end{equation}
    and $\vdelta=\left[\nabla_f \ell(f(\va_1; \vx^\star), b_1),\cdots, \nabla_f \ell(f(\va_n; \vx^\star), b_n)\right]^{\top}$ is the residual vector.

    Then, we partition $\vx^\star$ into two parts, i.e., $\vx^\star=[\vx^\star_I; \vx^\star_{I^c}]$ where $\vx^\star_I\in \cQ$ stands for the quantized part and $\vx^\star_{I^c}$ stands for the non-quantized part. Similarly, we partition the Jacobian matrix $\mJ$ as $\mJ=[\mJ_I; \mJ_{I^c}]$. Suppose that $|I|\geq d-n$, then we are done. Hence, we assume that $|I|< d-n$. Now let us consider the optimality condition over the index set $[d]- I$, which is given by 
    \begin{equation}
        -\frac{1}{\lambda}\mJ_{I^c}\vdelta\in \partial\Psi\left(\vx^\star_{I^c}\right).
    \end{equation}
    Since $\vx_{I^c}$ is the non-quantized part, we have $\partial\Psi\left(x_i^\star\right)\in \cA$ for all $i\in I^c$, which implies that 
    \begin{equation}
        -\frac{1}{\lambda}\mJ_{I^c}\vdelta\in \cA^{d-|I|}.
    \end{equation}
    Now, we show that this event happens with probability zero. To this end, for an arbitrary $\vv\in \cA^{d-|I|}$, we have 
    \begin{equation}
        \bP\left(-\frac{1}{\lambda}\mJ_{I^c}\vdelta=\vv\right)\leq \bP\left(\vv\in \col\left(\mJ_{I^c}\right)\right)\leq \bP\left(\det\left(\left[\mJ_{I^c},\vv\right]\right)=0\right).
    \end{equation}
    To proceed, according to \Cref{assumption::quantization-guarantee}, we know that the function
    \begin{equation}
        H(\va_1, \cdots, \va_n)=\det\left(\left[\mJ_{I^c},\vv\right]\right)=\det\left(\left[\nabla_{\vx} f(\va_1; \vx^\star), \cdots, \nabla_{\vx} f(\va_n; \vx^\star), \vv\right]_{I^c}\right)
    \end{equation}
    is also a real analytic function and $H(\va_1, \cdots, \va_n)\not\equiv 0$. As a consequence, the set $S=\{(\va_1, \cdots, \va_n): H(\va_1, \cdots, \va_n)=0\}$ has zero Lebesgue measure \cite{mityagin2015zero}. Additionally, since $\va_i\sim \cD_\va$ is absolutely continuous with respect to the $p$-dimensional Lebesgue measure, we have 
    \begin{equation}
        \bP\left(\det\left(\left[\mJ_{I^c},\vv\right]\right)=0\right)=\bP(S)=0.
    \end{equation}
Therefore, we show that with probability $1$, $-\frac{1}{\lambda}\mJ_{I^c}\vdelta\notin \cA^{d-|I|}$, which leads to the contradiction.
\end{proof}

\begin{proof}[Proof of \Cref{thm::main}]
    It suffices to check the \textit{Non-degenerate Jacobian} condition in \Cref{assumption::quantization-guarantee}. Note that
    \begin{equation}
        \left[\nabla_{\vx} f(\va_1; \vx^\star), \cdots, \nabla_{\vx} f(\va_n; \vx^\star), \vv\right]=\left[\va_1, \cdots, \va_n, \vv\right].
    \end{equation}
    Therefore, for any index set $I\subset [d]$ satisfying $|I|=n+1$ and $\vv\in \cA^d$, we can always choose $[\va_1, \cdots, \va_n]_I$ such that $\left[\va_1, \cdots, \va_n, \vv\right]_I$ is non-degenerate. This completes the proof.
\end{proof}

\section{Optimization algorithms}
\label{sec::optimization}
In this section, we first derive the proximal mappings for different variants of PARs and then introduce optimization algorithms that can leverage these mappings to efficiently solve the PARO problem.

\subsection{Proximal mapping of PARs}
\label{sec::proximal-mappings}
In this section, we study the \textbf{proximal mappings} for different variants of PARs, including convex, quasiconvex, and nonconvex formulations, which are defined by 
\begin{equation}
    \prox_{\lambda\Psi}(\vx)=\argmin_{\vz}\left\{\lambda\Psi(\vz)+\frac{1}{2}\norm{\vx-\vz}^2\right\}.
\end{equation}
Since we consider coordinate-wise PARs, the proximal mapping can be decomposed across coordinates, so we only need to analyze the scalar case
\begin{equation}
    \prox_{\lambda\Psi}(x)=\argmin_{z}\left\{\lambda\Psi(z)+\frac{1}{2}(x-z)^2\right\}.
\end{equation}

Proximal mapping is a key tool for optimizing regularized objectives and exhibits a strong connection to quantization with PAR: it often maps inputs to predefined discrete levels, effectively acting as a quantizer.

For any PAR, the proximal mapping is also piecewise affine, though it may exhibit discontinuities. Within each linear segment of the regularizer, the composite objective $\lambda\Psi(z)+\frac{1}{2}(x-z)^2$ becomes a strongly convex quadratic function, whose minimizer either coincides with an endpoint of the segment or corresponds to the unconstrained minimizer, which is also an affine function of $x$, provided it lies within the interval. In the remainder of this section, we derive explicit formulas for the proximal mappings of several representative PARs and visualize them in \Cref{fig::prox-mappings}. Notably, the proximal mappings act as (soft) quantizers, further revealing the quantization effect induced by PAR.

\paragraph{Convex PAR.}
Suppose that the set of target quantization values is given as
\(
\mathcal{Q} = \{0, \pm q_1,\ldots,\pm q_m\}\) 
with $0=q_0<q_1<\cdots<q_m$.
A convex PAR can be defined as
\begin{equation}\label{eqn:par-convex}
\Psi(x) = \max_{k\in\{0,\ldots,m\}} \{ a_k(|x|-q_k) + b_k \}, 
\end{equation}
where the slopes $\{a_k\}_{k=0}^m$ are free parameters that satisfy
$0 \leq a_0 < a_1 < \cdots < a_m = +\infty$, 
and $\{b_k\}_{k=0}^m$ are determined 
by setting $b_0=0$, and 
\begin{equation}
    b_k = b_{k-1} + a_{k-1}(q_k-q_{k-1}), \qquad k=1,\ldots,m.
\end{equation}
Its proximal mapping is provided by
    \begin{equation}
        \prox_{\lambda\Psi}(x) 
= \begin{cases}
\sign(x)\,q_k & \text{if}~|x|\in \left[\lambda a_{k-1}+q_k,\;\lambda a_k+q_k\right], \\
x-\sign(x)\lambda a_k & \text{if}~|x|\in\left[\lambda a_k+ q_k,\;\lambda a_k+ q_{k+1}\right].
\end{cases}
    \end{equation}
As shown in \Cref{fig::prox-mappings} (top row), when $x$ falls within the intervals $\pm\left[\lambda a_{k-1}+q_k,\;\lambda a_k+q_k\right]$, the proximal mapping quantizes it to the corresponding quantization value $\sign(x)q_k$. Moreover, as the regularization strength $\lambda$ increases, these quantization intervals become wider, making it more likely for the proximal mapping to produce a quantized solution when incorporated into an optimization algorithm.

\paragraph{Quasiconvex PAR.}
\begin{figure}
    \centering
     \begin{subfigure}[b]{0.3\textwidth}
         \centering
         \begin{tikzpicture}[scale=0.8]
    \tikzset{myarrow/.style={-{Straight Barb[scale=1]}, line width=\mylinewidthone}}
    \tikzset{thickline/.style={line width=\mylinewidthtwo}}
    
    \draw[myarrow] (-3,0) -- (3,0);
    \draw[myarrow] (0,0) -- (0,3);

    \draw[thickline]
        (-3,2) -- (-2,1.677) -- (-1,1) -- (0,0)
        -- (1,1) -- (2,1.677) -- (3,2);

    \draw[dashed, line width=1.2pt] (-2,0) -- (-2,1.677);
    \draw[dashed, line width=1.2pt] (-1,0) -- (-1,1);
    \draw[dashed, line width=1.2pt] (2,0) -- (2,1.677);
    \draw[dashed, line width=1.2pt] (1,0) -- (1,1);
    
    \node at (-2,-0.3) {\large $-q_2$};
    \node at (-1,-0.3) {\large $-q_1$};
    \node at (0,-0.3) {\large $0$};
    \node at (1,-0.3) {\large $q_1$};
    \node at (2,-0.3) {\large $q_2$};

\end{tikzpicture}
\caption{}
         \label{fig::quasiconvex-1}
     \end{subfigure}
     \hspace{2.5cm}
     \begin{subfigure}[b]{0.3\textwidth}
         \centering
         \begin{tikzpicture}[scale=0.8]
    \tikzset{myarrow/.style={-{Straight Barb[scale=1]}, line width=\mylinewidthone}}
    \tikzset{thickline/.style={line width=\mylinewidthtwo}}
    
    \draw[myarrow] (-3,0) -- (3,0);
    \draw[myarrow] (0,0) -- (0,3);

    \draw[thickline]
        (-3,2) -- (-2,1) -- (-1,1) -- (0,0)
        -- (1,1) -- (2,1) -- (3,2);

    \draw[dashed, line width=1.2pt] (-2,0) -- (-2,1);
    \draw[dashed, line width=1.2pt] (-1,0) -- (-1,1);
    \draw[dashed, line width=1.2pt] (2,0) -- (2,1);
    \draw[dashed, line width=1.2pt] (1,0) -- (1,1);
    
    \node at (-2,-0.3) {\large $-q_2$};
    \node at (-1,-0.3) {\large $-q_1$};
    \node at (0,-0.3) {\large $0$};
    \node at (1,-0.3) {\large $q_1$};
    \node at (2,-0.3) {\large $q_2$};

\end{tikzpicture}
\caption{}
         \label{fig::quasiconvex-2}
     \end{subfigure}
     \caption{Two different forms of quasiconvex PARs.}
    \label{fig::quasiconvex-par}
\end{figure}
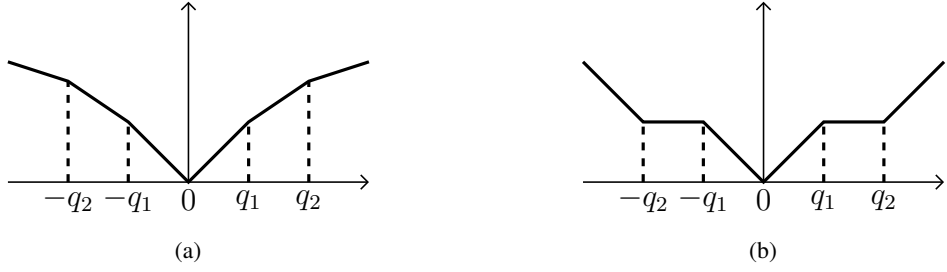
In this section, we study the proximal mapping for a class of quasiconvex PARs. A function $f:\bR\to \bR$ is said to be quasiconvex if, for any $x, y\in \bR$ and $\lambda\in [0, 1]$, it satisfies 
$$f(\lambda x+(1-\lambda)y)\leq \max\{f(x), f(y)\}.$$
\Cref{fig::quasiconvex-par} illustrates two representative quasiconvex PARs. In this paper, we primarily focus on the variant illustrated in \Cref{fig::quasiconvex-2}, as it empirically promotes quantization not only at zero but also at nonzero levels. In contrast, the version in \Cref{fig::quasiconvex-1} exhibits local concavity around all breakpoints except zero, which makes it less effective at inducing quantization to nonzero levels.

Although \Cref{fig::quasiconvex-2} technically violates the slope condition stated in \Cref{thm::main}, since it contains flat regions with zero slope, this issue is not problematic in practice. One can slightly perturb the zero slopes (for example, to $\pm \epsilon$ for a small $\epsilon > 0$) without significantly affecting the solution or its quantization guarantees. For clarity of exposition, we continue with the unperturbed version shown in \Cref{fig::quasiconvex-2}.

We consider the following quasiconvex PAR with uniformly spaced quantization gaps
\begin{equation}
\label{eq::psi-quasiconvex-general}
    \Psi(x)=\begin{cases}
        |x|-\frac{k}{2}q\quad \text{if}\quad kq\leq |x|\leq \frac{2k+1}{2}q,\\
        \frac{k+1}{2}q\quad \text{if}\quad\frac{2k+1}{2}q\leq |x|\leq (k+1)q.
    \end{cases}
\end{equation}
We focus on this equal-gap setting since general nonuniform cases yield more complex proximal mappings. We now characterize the proximal mapping for this PAR:
\begin{itemize}
    \item When the regularization strength $\lambda\leq q$, we have
    \begin{equation}
        \prox_{\lambda \Psi}(x)=\begin{cases}
        \sign(x)kq\quad \text{if}\quad kq\leq |x|\leq kq+\lambda,\\
        \sign(x)\left(|x|-\lambda\right)\quad \text{if}\quad kq+\lambda\leq |x|\leq \frac{2k+1}{2}q+\frac{\lambda}{2},\\
        \sign(x)|x|\quad \text{if}\quad \frac{2k+1}{2}q+\frac{\lambda}{2}\leq |x|\leq (k+1)q.
        \end{cases}
    \end{equation}
    \item When the regularization strength $\lambda\geq q$, we have
    \begin{equation}
        \prox_{\lambda \Psi}(x)=\sign(x)\left\lfloor \frac{|x|-\frac{\lambda}{2}}{q}\right\rceil q.
    \end{equation}
\end{itemize}
Unlike the convex PAR case, when $\lambda$ exceeds a certain threshold (e.g., $\lambda\geq q$), the proximal operator becomes a hard quantizer, mapping inputs exactly to discrete levels in the quantization set $\cQ$.

\begin{figure}[p]
    \centering
     \caption*{\textbf{Convex PAR}}
     \begin{subfigure}[b]{0.3\textwidth}
         \centering
         \begin{tikzpicture}[scale=0.55]
    \tikzset{myarrow/.style={-{Straight Barb[scale=1]}, line width=\mylinewidthone}}
    \tikzset{thickline/.style={line width=\mylinewidthtwo}}
    
    \draw[myarrow] (-4.8,0) -- (4.8,0);
    \draw[myarrow] (0,-3) -- (0,3);

    \draw[dashed, line width=0.9pt, gray] (-3,-3) -- (3, 3);
    \draw[dashed, line width=0.7pt, gray] (0,1) -- (1.4, 1);
    \draw[dashed, line width=0.7pt, gray] (0, 2) -- (2.8, 2);
    \draw[dashed, line width=0.7pt, gray] (0,-1) -- (-1.4, -1);
    \draw[dashed, line width=0.7pt, gray] (0, -2) -- (-2.8, -2);

    \draw[thickline]
        (-4.2, -2) -- (-2.4, -2) -- (-1.4, -1) -- (-1.2, -1) -- (-0.2, 0) -- (0,0)
        -- (0.2,0) -- (1.2, 1) -- (1.4, 1) -- (2.4, 2) -- (4.2, 2);

    \node at (0.7,-1) {\large $-q_1$};
    \node at (0.7,-2) {\large $-q_2$};
    \node at (0.4,-0.4) {\large $0$};
    \node at (-0.5,1) {\large $q_1$};
    \node at (-0.5,2) {\large $q_2$};

\end{tikzpicture}
          \caption{small $\lambda$}
         \label{fig:convex_lambda-0.2}
     \end{subfigure}
     \hfill
     \begin{subfigure}[b]{0.3\textwidth}
         \centering
         \begin{tikzpicture}[scale=0.55]
    \tikzset{myarrow/.style={-{Straight Barb[scale=1]}, line width=\mylinewidthone}}
    \tikzset{thickline/.style={line width=\mylinewidthtwo}}
    
    \draw[myarrow] (-4.8,0) -- (4.8,0);
    \draw[myarrow] (0,-3) -- (0,3);

    \draw[thickline]
        (-4.2, -2) -- (-3.2, -2) -- (-2.2, -1) -- (-1.6, -1) -- (-0.6, 0) -- (0,0)
        -- (0.6,0) -- (1.6, 1) -- (2.2, 1) -- (3.2, 2) -- (4.2, 2);

    \draw[dashed, line width=0.9pt, gray] (-3,-3) -- (3, 3);
    \draw[dashed, line width=0.7pt, gray] (0,1) -- (1.4, 1);
    \draw[dashed, line width=0.7pt, gray] (0, 2) -- (3.2, 2);
    \draw[dashed, line width=0.7pt, gray] (0,-1) -- (-1.4, -1);
    \draw[dashed, line width=0.7pt, gray] (0, -2) -- (-3.2, -2);
    \node at (0.7,-1) {\large $-q_1$};
    \node at (0.7,-2) {\large $-q_2$};
    \node at (0.4,-0.4) {\large $0$};
    \node at (-0.5,1) {\large $q_1$};
    \node at (-0.5,2) {\large $q_2$};

\end{tikzpicture}
          \caption{medium $\lambda$}
         \label{fig:convex_lambda-0.8}
     \end{subfigure}
     \hfill
     \begin{subfigure}[b]{0.3\textwidth}
         \centering
         \begin{tikzpicture}[scale=0.55]
    \tikzset{myarrow/.style={-{Straight Barb[scale=1]}, line width=\mylinewidthone}}
    \tikzset{thickline/.style={line width=\mylinewidthtwo}}
    
    \draw[myarrow] (-4.8,0) -- (4.8,0);
    \draw[myarrow] (0,-3) -- (0,3);

    \draw[thickline]
        (-4.5, -2) -- (-4, -2) -- (-3, -1) -- (-2, -1) -- (-1, 0) -- (0,0)
        -- (1,0) -- (2, 1) -- (3, 1) -- (4, 2) -- (4.5, 2);

    \draw[dashed, line width=0.9pt, gray] (-3,-3) -- (3, 3);
    \draw[dashed, line width=0.7pt, gray] (0,1) -- (2, 1);
    \draw[dashed, line width=0.7pt, gray] (0, 2) -- (4, 2);
    \draw[dashed, line width=0.7pt, gray] (0,-1) -- (-2, -1);
    \draw[dashed, line width=0.7pt, gray] (0, -2) -- (-4, -2);
    \node at (0.7,-1) {\large $-q_1$};
    \node at (0.7,-2) {\large $-q_2$};
    \node at (0.4,-0.4) {\large $0$};
    \node at (-0.5,1) {\large $q_1$};
    \node at (-0.5,2) {\large $q_2$};

\end{tikzpicture}
          \caption{large $\lambda$}
         \label{fig:convex_lambda-1}
     \end{subfigure}\\[4ex]
     \caption*{\textbf{Quasiconvex PAR}}
     \begin{subfigure}[b]{0.3\textwidth}
         \centering
         \begin{tikzpicture}[scale=0.4]
    \tikzset{myarrow/.style={-{Straight Barb[scale=1]}, line width=\mylinewidthone}}
    \tikzset{thickline/.style={line width=\mylinewidthtwo}}
    
    \draw[myarrow] (-5.8,0) -- (5.8,0);
    \draw[myarrow] (0,-4.5) -- (0,4.5);

    \draw[dashed, line width=0.9pt, gray] (-4,-4) -- (4, 4);
    \draw[dashed, line width=0.7pt, gray] (-2,-2) -- (-2, 0);
    \draw[dashed, line width=0.7pt, gray] (2,0) -- (2, 2);
    \draw[dashed, line width=0.7pt, gray] (-4,-4) -- (-4, 0);
    \draw[dashed, line width=0.7pt, gray] (4,0) -- (4, 4);
    \draw[dashed, line width=0.7pt, gray] (-2,-2) -- (0, -2);
    \draw[dashed, line width=0.7pt, gray] (0,2) -- (2, 2);
    \draw[dashed, line width=0.7pt, gray] (-4,-4) -- (0, -4);
    \draw[dashed, line width=0.7pt, gray] (0, 4) -- (4, 4);

    \draw[thickline]
        (-4, -4) -- (-3.2, -3.2);
    \draw[thickline]
        (-3.2, -2.8) -- (-2.4, -2) -- (-2, -2) -- (-1.2, -1.2);
    \draw[thickline]
         (-1.2, -0.8) -- (-0.4, 0) -- (0,0)
        -- (0.4,0) -- (1.2, 0.8);
    \draw[thickline]
         (1.2, 1.2) -- (2, 2) -- (2.4, 2) -- (3.2, 2.8);
    \draw[thickline]
         (3.2, 3.2) -- (4, 4);
    \draw[thickline, dashed]
         (-3.2, -3.2) -- (-3.2, -2.8);
    \draw[thickline, dashed]
         (-1.2, -1.2) -- (-1.2, -0.8);
    \draw[thickline, dashed]
         (1.2, 0.8) -- (1.2, 1.2);
    \draw[thickline, dashed]
         (3.2, 2.8) -- (3.2, 3.2);

    \filldraw [black] (-3.2, -3.2) circle (3.5pt);
    \filldraw [black] (-3.2, -2.8) circle (3.5pt);
    \filldraw [black] (-1.2, -1.2) circle (3.5pt);
    \filldraw [black] (-1.2, -0.8) circle (3.5pt);
    \filldraw [black] (3.2, 3.2) circle (3.5pt);
    \filldraw [black] (3.2, 2.8) circle (3.5pt);
    \filldraw [black] (1.2, 1.2) circle (3.5pt);
    \filldraw [black] (1.2, 0.8) circle (3.5pt);

    \node at (-2.9,-0.6) {\large $-q_1$};
    \node at (-4.9,-0.6) {\large $-q_2$};
    \node at (0.4,-0.6) {\large $0$};
    \node at (2,-0.6) {\large $q_1$};
    \node at (4,-0.6) {\large $q_2$};
    \node at (1,-2) {\large $-q_1$};
    \node at (1,-4) {\large $-q_2$};
    \node at (-0.7,2) {\large $q_1$};
    \node at (-0.7,4) {\large $q_2$};

\end{tikzpicture}
          \caption{small $\lambda$}
         \label{fig:quasiconvex_lambda-0.2}
     \end{subfigure}
     \hfill
     \begin{subfigure}[b]{0.3\textwidth}
         \centering
         \begin{tikzpicture}[scale=0.4]
    \tikzset{myarrow/.style={-{Straight Barb[scale=1]}, line width=\mylinewidthone}}
    \tikzset{thickline/.style={line width=\mylinewidthtwo}}
    
    \draw[myarrow] (-5.8,0) -- (5.8,0);
    \draw[myarrow] (0,-4.5) -- (0,4.5);

    \draw[dashed, line width=0.9pt, gray] (-4,-4) -- (4, 4);
    \draw[dashed, line width=0.7pt, gray] (-2,-2) -- (-2, 0);
    \draw[dashed, line width=0.7pt, gray] (2,0) -- (2, 2);
    \draw[dashed, line width=0.7pt, gray] (-4,-2) -- (-4, 0);
    \draw[dashed, line width=0.7pt, gray] (4,0) -- (4, 2);
    \draw[dashed, line width=0.7pt, gray] (-2,-2) -- (0, -2);
    \draw[dashed, line width=0.7pt, gray] (0,2) -- (2, 2);
    \draw[dashed, line width=0.7pt, gray] (-4,-4) -- (0, -4);
    \draw[dashed, line width=0.7pt, gray] (0, 4) -- (4, 4);
    \draw[thickline]
        (-5, -4) -- (-4, -4);
    \draw[thickline]
        (-4, -2) -- (-2, -2);
    \draw[thickline]
        (-2,0) -- (0,0)
        -- (2,0);
    \draw[thickline]
        (2, 2) -- (4, 2);
    \draw[thickline]
        (5, 4) -- (4, 4);

    \draw[thickline, dashed]
         (-4, -4) -- (-4, -2);
    \draw[thickline, dashed]
         (-2, -2) -- (-2,0);
    \draw[thickline, dashed]
         (2,0) -- (2, 2);
    \draw[thickline, dashed]
         (4, 2) -- (4, 4);

    \filldraw [black] (-4, -4) circle (3.5pt);
    \filldraw [black] (-4, -2) circle (3.5pt);
    \filldraw [black] (-2, -2) circle (3.5pt);
    \filldraw [black] (-2, 0) circle (3.5pt);
    \filldraw [black] (4, 4) circle (3.5pt);
    \filldraw [black] (4, 2) circle (3.5pt);
    \filldraw [black] (2, 2) circle (3.5pt);
    \filldraw [black] (2, 0) circle (3.5pt);

    \node at (-2.9,-0.6) {\large $-q_1$};
    \node at (-4.9,-0.6) {\large $-q_2$};
    \node at (0.4,-0.6) {\large $0$};
    \node at (2,-0.6) {\large $q_1$};
    \node at (4,-0.6) {\large $q_2$};
    \node at (1,-2) {\large $-q_1$};
    \node at (1,-4) {\large $-q_2$};
    \node at (-0.7,2) {\large $q_1$};
    \node at (-0.7,4) {\large $q_2$};

\end{tikzpicture}
          \caption{medium $\lambda$}
         \label{fig:quasiconvex_lambda-0.8}
     \end{subfigure}
     \hfill
     \begin{subfigure}[b]{0.3\textwidth}
         \centering
         \begin{tikzpicture}[scale=0.4]
    \tikzset{myarrow/.style={-{Straight Barb[scale=1]}, line width=\mylinewidthone}}
    \tikzset{thickline/.style={line width=\mylinewidthtwo}}
    
    \draw[myarrow] (-5.8,0) -- (5.8,0);
    \draw[myarrow] (0,-4.5) -- (0,4.5);

    \draw[dashed, line width=0.9pt, gray] (-4,-4) -- (4, 4);
    \draw[dashed, line width=0.7pt, gray] (-3,-2) -- (0, -2);
    \draw[dashed, line width=0.7pt, gray] (0,2) -- (3, 2);
    \draw[dashed, line width=0.7pt, gray] (-5,-4) -- (0, -4);
    \draw[dashed, line width=0.7pt, gray] (0, 4) -- (5, 4);
    \draw[thickline]
        (-5.5, -4) -- (-5, -4);
    \draw[thickline]
        (-5, -2) -- (-3, -2);
    \draw[thickline]
        (-3,0) -- (3,0);
    \draw[thickline]
        (3, 2) -- (5, 2);
    \draw[thickline]
        (5, 4) -- (5.5, 4);

    \draw[thickline, dashed]
         (-5, -4) -- (-5, -2);
    \draw[thickline, dashed]
         (-3, -2) -- (-3,0);
    \draw[thickline, dashed]
         (3,0) -- (3, 2);
    \draw[thickline, dashed]
         (5, 2) -- (5, 4);

    \filldraw [black] (-5, -4) circle (3.5pt);
    \filldraw [black] (-5, -2) circle (3.5pt);
    \filldraw [black] (-3, -2) circle (3.5pt);
    \filldraw [black] (-3, 0) circle (3.5pt);
    \filldraw [black] (5, 4) circle (3.5pt);
    \filldraw [black] (5, 2) circle (3.5pt);
    \filldraw [black] (3, 2) circle (3.5pt);
    \filldraw [black] (3, 0) circle (3.5pt);

    \node at (-2.9,-0.6) {\large $-q_1$};
    \node at (-4.9,-0.6) {\large $-q_2$};
    \node at (0.4,-0.6) {\large $0$};
    \node at (2,-0.6) {\large $q_1$};
    \node at (4,-0.6) {\large $q_2$};
    \node at (1,-2) {\large $-q_1$};
    \node at (1,-4) {\large $-q_2$};
    \node at (-0.7,2) {\large $q_1$};
    \node at (-0.7,4) {\large $q_2$};

\end{tikzpicture}
          \caption{large $\lambda$}
         \label{fig:quasiconvex_lambda-1}
     \end{subfigure}\\[4ex]
    \caption*{\textbf{Nonconvex PAR}}
    \begin{subfigure}[b]{0.3\textwidth}
         \centering
         \begin{tikzpicture}[scale=0.65]
    \tikzset{myarrow/.style={-{Straight Barb[scale=1]}, line width=\mylinewidthone}}
    \tikzset{thickline/.style={line width=\mylinewidthtwo}}
    
    \draw[myarrow] (-3.5,0) -- (3.5,0);
    \draw[myarrow] (0,-3) -- (0,3);

    \draw[dashed, line width=0.7pt, gray] (-3, -3) -- (3, 3);

    \draw[dashed, line width=0.7pt, gray] (-2,-2) -- (-2, 0);
    \draw[dashed, line width=0.7pt, gray] (2, 0) -- (2, 2);
    \draw[dashed, line width=0.7pt, gray] (-2,-2) -- (0, -2);
    \draw[dashed, line width=0.7pt, gray] (0, 2) -- (2, 2);

    \draw[thickline]
        (-3, -2.8) -- (-2.2, -2) -- (-1.8, -2) -- (-1, -1.2);
    \draw[thickline]
        (-1, -0.8) -- (-0.2,0) -- (0,0)
        -- (0.2,0) -- (1, 0.8);
    \draw[thickline]
        (1, 1.2) -- (1.8, 2) -- (2.2, 2) -- (3, 2.8);
    \draw[thickline, dashed]
        (-1, -1.2) -- (-1, -0.8);
    \draw[thickline, dashed]
        (1, 1.2) -- (1, 0.8);

    \filldraw [black] (-3, -2.8) circle (2.5pt);
    \filldraw [black] (-1, -1.2) circle (2.5pt);
    \filldraw [black] (-1, -0.8) circle (2.5pt);
    \filldraw [black] (3, 2.8) circle (2.5pt);
    \filldraw [black] (1, 1.2) circle (2.5pt);
    \filldraw [black] (1, 0.8) circle (2.5pt);
    
    \node at (-2.7,-0.4) {\large $-q_1$};
    \node at (0.7,-2) {\large $-q_1$};
    \node at (0.4,-0.4) {\large $0$};
    \node at (2,-0.4) {\large $q_1$};
    \node at (-0.4,2) {\large $q_1$};

\end{tikzpicture}
          \caption{small $\lambda$}
         \label{fig:nonconvex_lambda-0.2}
     \end{subfigure}
     \hfill
     \begin{subfigure}[b]{0.3\textwidth}
         \centering
         \begin{tikzpicture}[scale=0.65]
    \tikzset{myarrow/.style={-{Straight Barb[scale=1]}, line width=\mylinewidthone}}
    \tikzset{thickline/.style={line width=\mylinewidthtwo}}
    
    \draw[myarrow] (-3.5,0) -- (3.5,0);
    \draw[myarrow] (0,-3) -- (0,3);

    \draw[dashed, line width=0.9pt, gray] (-3,-3) -- (3, 3);
    \draw[dashed, line width=0.7pt, gray] (-2,-2) -- (-2, 0);
    \draw[dashed, line width=0.7pt, gray] (2,0) -- (2, 2);
    \draw[dashed, line width=0.7pt, gray] (-2,-2) -- (0, -2);
    \draw[dashed, line width=0.7pt, gray] (0, 2) -- (2, 2);

    \draw[thickline]
        (-3, -2.2) -- (-2.8, -2) -- (-1.2, -2) -- (-1, -1.8);
    \draw[thickline]
        (-1, -0.2) -- (-0.8,0) -- (0,0)
        -- (0.8,0) -- (1, 0.2);
    \draw[thickline]
        (1, 1.8) -- (1.2, 2) -- (2.8, 2) -- (3, 2.2);
        
    \draw[thickline, dashed]
        (-1, -1.8) -- (-1, -0.2);
    \draw[thickline, dashed]
        (1, 1.8) -- (1, 0.2);
        
    \filldraw [black] (-3, -2.2) circle (2.5pt);
    \filldraw [black] (-1, -1.8) circle (2.5pt);
    \filldraw [black] (-1, -0.2) circle (2.5pt);
    \filldraw [black] (3, 2.2) circle (2.5pt);
    \filldraw [black] (1, 1.8) circle (2.5pt);
    \filldraw [black] (1, 0.2) circle (2.5pt);

    \node at (-2.7,-0.4) {\large $-q_1$};
    \node at (0.7,-2) {\large $-q_1$};
    \node at (0.4,-0.4) {\large $0$};
    \node at (2,-0.4) {\large $q_1$};
    \node at (-0.4,2) {\large $q_1$};

\end{tikzpicture}
          \caption{medium $\lambda$}
         \label{fig:nonconvex_lambda-0.8}
     \end{subfigure}
     \hfill
     \begin{subfigure}[b]{0.3\textwidth}
         \centering
         \begin{tikzpicture}[scale=0.65]
    \tikzset{myarrow/.style={-{Straight Barb[scale=1]}, line width=\mylinewidthone}}
    \tikzset{thickline/.style={line width=\mylinewidthtwo}}
    
    \draw[myarrow] (-3.5,0) -- (3.5,0);
    \draw[myarrow] (0,-3) -- (0,3);

    \draw[dashed, line width=0.9pt, gray] (-3,-3) -- (3, 3);
    \draw[dashed, line width=0.7pt, gray] (-2,-2) -- (-2, 0);
    \draw[dashed, line width=0.7pt, gray] (2,0) -- (2, 2);
    \draw[dashed, line width=0.7pt, gray] (-2,-2) -- (0, -2);
    \draw[dashed, line width=0.7pt, gray] (0, 2) -- (2, 2);

    \draw[thickline]
        (-3, -2) -- (-1, -2);
    \draw[thickline]
        (-1,0) -- (0,0)
        -- (1,0);
    \draw[thickline]
        (1, 2) -- (3, 2);
    \draw[thickline, dashed]
        (-1, -2) -- (-1, -0);
    \draw[thickline, dashed]
        (1, 2) -- (1, 0);
    
    \filldraw [black] (-3, -2) circle (2.5pt);
    \filldraw [black] (-1, -2) circle (2.5pt);
    \filldraw [black] (-1, 0) circle (2.5pt);
    \filldraw [black] (3, 2) circle (2.5pt);
    \filldraw [black] (1, 2) circle (2.5pt);
    \filldraw [black] (1, 0) circle (2.5pt);

    \node at (-2.7,-0.4) {\large $-q_1$};
    \node at (0.7,-2) {\large $-q_1$};
    \node at (0.4,-0.4) {\large $0$};
    \node at (2,-0.4) {\large $q_1$};
    \node at (-0.4,2) {\large $q_1$};

\end{tikzpicture}
          \caption{large $\lambda$}
         \label{fig:nonconvex_lambda-1}
     \end{subfigure}\\[4ex]
     \caption{Proximal mappings for different PARs. Each row corresponds to a class of PARs: convex (top), quasiconvex (middle), and nonconvex (bottom). Each column illustrates the proximal mapping under a different regularization strength: small, medium, and large~$\lambda$.}
    \label{fig::prox-mappings}
\end{figure}

\paragraph{Nonconvex PAR.}
Given the quantization values $\cQ=\{q_1, \cdots, q_m\}$ satisfying $q_1<q_2<\cdots <q_m$ (here we allow general asymmetric quantization values), we consider the following nonconvex PAR 
\begin{equation}
\label{eq::psi-nonconvex}
    \Psi(x)=\begin{cases}
        x-q_k\quad \text{if }q_k\leq x\leq \frac{q_k+q_{k+1}}{2},\\
        -x+q_{k+1}\quad \text{if }\frac{q_k+q_{k+1}}{2}\leq x\leq q_{k+1}.
    \end{cases}
\end{equation}
Its proximal mapping $\prox_{\lambda\Psi}(\cdot)$ admits the following closed-form solution:
\begin{equation}
        \prox_{\lambda\Psi}(x)=\begin{cases}
        \clip\left(x-\lambda, q_k, \frac{q_k+q_{k+1}}{2} \right)\quad \text{if }q_k\leq x\leq \frac{q_k+q_{k+1}}{2},\\
        \clip\left(x+\lambda, \frac{q_k+q_{k+1}}{2},  q_{k+1}\right)\quad \text{if }\frac{q_k+q_{k+1}}{2}\leq x\leq q_{k+1}.
        \end{cases}
    \end{equation}
    Here the clip function $\clip(x, a, b)$ is defined by $\clip(x, a, b)=\min\{\max\{x, a\}, b\}$. In particular, if $\lambda\geq \frac{1}{2}\max_{1\leq k\leq m-1}(q_{k+1}-q_k)$, the proximal mapping reduces to an exact quantization mapping 
    \begin{equation}
        \prox_{\lambda\Psi}(x)=\argmin_{q\in \cQ}|x-q|.
    \end{equation}
Similar to the quasiconvex PAR, this nonconvex PAR induces hard quantization once $\lambda$ exceeds a certain threshold. However, unlike the quasiconvex case, where the input magnitude is first reduced before snapping to the nearest quantization level, the proximal mapping of the nonconvex PAR remains fixed once $\lambda$ surpasses the threshold.

Another interesting feature to notice is that (see \Cref{fig::prox-mappings}), for convex and nonconvex PARs, the magnigude of their output is no larger than that of the input; but for nonconvex PARs, this is no longer true.

\subsection{Proximal gradient method}
\label{sec::PG}
To minimize the PARO objective $F_{\lambda}(\vx)=f(\vx)+\lambda\Psi(\vx)$ where $f(\cdot)$ is a smooth loss function and $\Psi(\cdot)$ is a PAR, 
we first consider the classic proximal gradient method 
\begin{equation}\label{eq::prox-grad-algm}
    \vx^{t+1}=\prox_{\eta_t\lambda\Psi}\left(\vx^{t}-\eta_t\nabla f(\vx^t)\right).
\end{equation}
Here $\eta_t>0$ is the stepsize.
In practice, the stepsize $\eta_t$ is often selected adaptively using a line search scheme. 
We refer the reader to \cite{parikh2014proximal,beck2017book,wright_recht2022book} for a comprehensive overview of proximal gradient methods.
The following result gives the convergence guarantee of the proximal gradient algorithm (\ref{eq::prox-grad-algm}) for PARO.

\begin{theorem}
    \label{thm::convergence-prox-gd}
    Suppose $f(\cdot)$ is $L$-smooth and the stepsize $\eta_t\equiv\eta\leq \frac{1}{2L}$. Then, the following arguments hold
    \begin{itemize}
        \item \textbf{Convex case:} If both $f(\cdot)$ and $\Psi(\cdot)$ are convex, then we have 
        \begin{equation}
            F_{\lambda}\left(\vx^{T}\right)-F_{\lambda}^{\star} \leq \frac{\norm{\vx^{0}-\vx^{\star}}^2}{2 \eta T}.
        \end{equation}
        \item \textbf{General case:} If $F_{\lambda}(\vx)=f(x)+\lambda\Psi(\vx)$ is proper and coercive, then the iterates $\{\vx^t\}_{t=0}^{\infty}$ generated by (\ref{eq::prox-grad-algm}) are bounded. Moreover, any accumulation point $\vx^\star$ of $\{\vx^t\}$ satisfies $\bm{0} \in \partial F_\lambda(\vx^\star)$, i.e., $\vx^\star$ is a critical point.
    \end{itemize}
\end{theorem}

These results demonstrate that the proximal gradient method efficiently converges to a critical point of the regularized loss function $F_{\lambda}(\cdot)$, which corresponds to the global minimum in the convex case. This is particularly desirable since, under mild conditions, all critical points in general supervised learning problems are highly quantized (\Cref{thm::main-extended}). 
Combining these two findings, we conclude that the proximal gradient method efficiently converges to a highly quantized solution.
\begin{proof}[Proof of \Cref{thm::convergence-prox-gd}]
    The convex case follows directly from \cite[Theorem 3.1]{beck2009fast}. For the general nonconvex case, our proof parallels the argument used in the proof of Theorem 1 in \cite{li2015accelerated}. We include it here for completeness.
    
    To begin, observe that
    \begin{equation}
        \begin{aligned}
            &F_{\lambda}(\vx^{t+1})\\&=f(\vx^{t+1})+\lambda \Psi(\vx^{t+1})\\&\stackrel{(a)}{\leq} f(\vx^t)+\inner{\nabla f(\vx^t)}{\vx^{t+1}-\vx^t}+\frac{L}{2}\norm{\vx^{t+1}-\vx^t}^2+\lambda\Psi(\vx^{t+1})\\
            &= f(\vx^t)+\inner{\nabla f(\vx^t)}{\vx^{t+1}-\vx^t}+\frac{1}{2\eta}\norm{\vx^{t+1}-\vx^t}^2+\lambda\Psi(\vx^{t+1})+\left(\frac{L}{2}-\frac{1}{2\eta}\right)\norm{\vx^{t+1}-\vx^t}^2\\
            &\stackrel{(b)}{\leq} f(\vx^t)+\lambda\Psi(\vx^t)+\left(\frac{L}{2}-\frac{1}{2\eta}\right)\norm{\vx^{t+1}-\vx^t}^2\\
            &\stackrel{(c)}{=} F_{\lambda}(\vx^t)-\frac{1}{4\eta}\norm{\vx^{t+1}-\vx^t}^2.
        \end{aligned}
        \label{eq::31}
    \end{equation}
    Here $(a)$ comes from the fact that $f(\cdot)$ is $L$-smooth; $(b)$ is due to the definition of $\vx^{t+1}$; and in $(c)$ we use the condition that $\eta\leq\frac{1}{2L}$. 
    As a result, $\{F_{\lambda}(\vx^t)\}$ is nonincreasing. Since $\{F_{\lambda}(\cdot)\}$ is assumed to be proper and coercive, the iterates $\{\vx^t\}$ is bounded. Hence, $\{\vx^t\}$ has at least one accumulation point.
    
    Next, we show that any accumulation point $\vx^{\star}$ is a critical point of $F_{\lambda}(\cdot)$. By telescoping \eqref{eq::31}, we obtain
    \begin{equation}
        \frac{1}{T}\sum_{t=0}^{T-1}\norm{\vx^{t+1}-\vx^t}^2\leq \frac{\left(F_{\lambda}(\vx^0)-F_{\lambda}^{\star}\right)}{4\eta T},
    \end{equation}
    which implies $\norm{\vx^{t+1}-\vx^t}\to 0$ as $t\to \infty$.
    From the proximal update rule,
    \begin{equation}
        \vx^{t+1}\in \argmin_{\vz}\left\{\eta_t\lambda\Psi(\vz)+\frac{1}{2}\norm{\vz - \left(\vx^{t}-\eta_t\nabla f(\vx^{t})\right)}^2\right\}.
    \end{equation}
    The optimality condition gives 
    \begin{equation}\label{eq::prox-opt-cond}
        \vzero \in \left(\eta_t\lambda\partial \Psi(\vx^{t+1})+\vx^{t+1} - \left(\vx^{t}-\eta_t\nabla f(\vx^{t})\right)\right).
    \end{equation}
    Let's define
    \[
    \vv^t := \vx^t-\vx^{t+1}+\eta_t\left(\nabla f(\vx^{t})-\nabla f(\vx^{t+1})\right),
    \]
    then~(\ref{eq::prox-opt-cond}) implies
    \begin{equation}
        \vv^t \in \partial F_{\lambda}(\vx^{t+1}).
    \end{equation}
    Using the Lipschitz continuity of $\nabla f$, we have
    \begin{align*}
        \norm{\vv^t} &= 
        \norm{\vx^t-\vx^{t+1}+\eta_t\left(\nabla f(\vx^{t})-\nabla f(\vx^{t+1})\right)} \\
        &\leq \left(1+\eta_t L\right)\norm{\vx^t-\vx^{t+1}}\\
        &=\frac{3}{2}\norm{\vx^t-\vx^{t+1}},
    \end{align*}
   which converges to $0$ as $t\to\infty$.
   Therefore, $\norm{\vv^t}\to 0$ as $t\to\infty$.
   
    For any accumulation point $\vx^\star$, there exists a subsequence $\{\vx^{t_k}\}$ such that $\vx^{t_k}\to \vx^\star$. Since $\vv^{t_k}\in \partial F_{\lambda}(\vx^{t_k})$ and $\norm{\vv^{t_k}} \to 0$, 
    we conclude that $\bm{0}\in \partial F_{\lambda}(\vx^\star)$ by the closedness of the limiting subdifferential.
    This completes the proof.
    
\end{proof}

\paragraph{Accelerated proximal gradient methods.}
To further accelerate convergence, one may apply the accelerated proximal gradient method, which incorporates a momentum term
\begin{equation}
    \begin{aligned} 
    \vy^{t+1} &=\vx^t+\beta^t\left(\vx^t-\vx^{t-1}\right), \\ 
    \vx^{t+1} &=\prox_{\eta_t\lambda\Psi}\left(\vy^{t+1}-\eta_t \nabla f\left(\vy^{t+1}\right)\right).\end{aligned}
\end{equation}
Here $\beta_t\in (0, 1)$ is the momentum coefficient and $\eta_t>0$ is the stepsize. See \cite[Section~4.3]{parikh2014proximal} for more details about the choices of the stepsize and momentum coefficient. 
In the convex setting, classical results show that the accelerated method achieves a convergence rate of $O(1/T^2)$ for the regularized objective when appropriate hyperparameters are used \cite{vandenberghe2010fast,beck2010gradient}. This rate extends to nonconvex objectives with convex penalties, as shown in \cite{nesterov2013gradient}. For general nonconvex problems, more sophisticated variants have been developed; see \cite{li2015accelerated}.

\subsection{Alternating direction method of multipliers (ADMM)}
\label{sec::ADMM}
When the loss function $f(\cdot)$ exhibits certain structures, such as in linear or logistic regression, the alternating direction method of multipliers (ADMM) offers an effective alternative to the proximal gradient method, often yielding faster convergence in practice \cite{boyd2011distributed}.

To apply ADMM, we reformulate the original problem $\min_{\vx}F_{\lambda}(\vx)=f(\vx)+\lambda \Psi(\vx)$ as the constrained problem 
\begin{equation}
    \min_{\vx, \vz} f(\vx)+\lambda\Psi(\vz) \quad\text{such that}\quad \vx-\vz=0.
\end{equation}
The corresponding augmented Lagrangian with penalty parameter $\rho$ is given by
\begin{equation}
L_{\rho}(\vx, \vz, \vy)=f(\vx)+\lambda\Psi(\vz)+\rho\inner{\vy}{\vx-\vz}+\frac{\rho}{2}\norm{\vx-\vz}^2,
\end{equation}
where $\vy$ is the dual variable. ADMM performs the following updates at each iteration 
\begin{align}
        \vx^{t+1}&=\argmin_{\vx}L_{\rho}(\vx, \vz^t, \vy^t),\tag{$x$-minimization}\label{x-minimization}\\
        \vz^{t+1}&=\argmin_{\vz}L_{\rho}(\vx^{t+1}, \vz, \vy^t)=\prox_{(\lambda/\rho)\cdot\Psi}\left(\vx^{t+1}+\vy^t\right),\tag{$z$-minimization}\\
        \vy^{t+1}&=\vy^t+\rho\left(\vx^{t+1}-\vz^{t+1}\right)\tag{dual update}.
    \end{align}
When $f$ admits a convenient structure, the \ref{x-minimization} can often be computed efficiently. For example, if $f(\vx)=\frac{1}{2}\norm{\mA\vx-\vb}^2$, it admits a closed-form solution $\vx^{t+1}=\left(\mA^{\top}\mA+\rho \mI\right)^{-1}\left(\mA^{\top} \vb+\rho\left(\vz^{t}-\vy^{t}\right)\right)$. ADMM is known to achieve convergence rates comparable to the proximal gradient method in both convex and certain nonconvex settings \cite{yang2016linear}. Furthermore, when additional structural conditions are satisfied, including linear and logistic regression with convex PARs as special cases, ADMM enjoys a linear convergence rate \cite{hong2017linear}. In \Cref{sec::numerical-experiments}, we empirically compare the performance of ADMM against other benchmark methods on linear regression tasks.

\section{Statistical guarantees of PAR for linear regression}
\label{sec::statistical-guarantee}
In practice, we seek quantized solutions that not only achieve low training loss but also exhibit strong statistical guarantees. This section investigates the statistical properties of PAR-regularized solutions in the context of linear regression. Our main result shows that specific PAR formulations closely approximate widely used regularizers, including $\ell_2$-, $\ell_1$-, and more general nonconvex regularizers; see \Cref{fig::par-as-different-regularizers} for an illustration. As a result, the corresponding PAR-regularized estimators inherit similar statistical guarantees as their classical counterparts. 

Specifically, we consider the following PAR-regularized linear regression problem:
\begin{equation}
    \min_{\vx\in \bR^d} F_{\PAR}(\vx)=\frac{1}{2n}\norm{\mA\vx-\vb}^2+\lambda \Psi(\vx).
    \label{eq::linear-regression}
\end{equation}
Here, $\mA\in \bR^{n\times d}$ is the design matrix and $\vb\in \bR^n$ is the response vector. We will specify the data model and the PAR formulation in each subsection. Throughout this section, we denote the PAR-regularized solution by $\vx^{\star}_{\PAR}=\argmin_{\vx\in \bR^d}F_{\PAR}(\vx)$.

\subsection{PAR as ridge regression}
\begin{figure}
    \centering
    \begin{subfigure}[b]{0.28\textwidth}
         \centering
         \begin{tikzpicture}[scale=0.8]
    \tikzset{myarrow/.style={-{Straight Barb[scale=1]}, line width=\mylinewidthone}}
    \tikzset{thickline/.style={line width=1.2pt}}
    
    \draw[myarrow] (-3,0) -- (3,0);
    \draw[myarrow] (0,0) -- (0,3);

    \draw[domain=-2.5:2.5, smooth, variable=\x, blue, thick] 
        plot ({\x}, {0.3333*\x*\x});
    \draw[thickline]
        (-2.5,2.08333) -- (-2,1.333) -- (-1,0.333) -- (0,0)
        -- (1,0.333) -- (2,1.333) -- (2.5,2.08333);

    \filldraw [black] (0,0) circle (1.2pt);
    \filldraw [black] (-2,1.3333) circle (1.2pt);
    \filldraw [black] (2,1.3333) circle (1.2pt);
    \filldraw [black] (-1,0.3333) circle (1.2pt);
    \filldraw [black] (1,0.3333) circle (1.2pt);
    \draw[dashed, line width=1.2pt] (-2,0) -- (-2,1.3333);
    \draw[dashed, line width=1.2pt] (-1,0) -- (-1,0.333);
    \draw[dashed, line width=1.2pt] (2,0) -- (2,1.333);
    \draw[dashed, line width=1.2pt] (1,0) -- (1,0.333);
    
    \node at (-2,-0.3) {\large $-q_2$};
    \node at (-1,-0.3) {\large $-q_1$};
    \node at (0,-0.3) {\large $0$};
    \node at (1,-0.3) {\large $q_1$};
    \node at (2,-0.3) {\large $q_2$};

\end{tikzpicture}
         \caption{PAR as $x^2$}
         \label{fig::PAR-as-l2}
     \end{subfigure}
     \hfill
     \begin{subfigure}[b]{0.28\textwidth}
         \centering
         \begin{tikzpicture}[scale=0.8]
    \tikzset{myarrow/.style={-{Straight Barb[scale=1]}, line width=\mylinewidthone}}
    \tikzset{thickline/.style={line width=1.2pt}}
    
    \draw[myarrow] (-3,0) -- (3,0);
    \draw[myarrow] (0,0) -- (0,3);

    \draw[domain=-2.5:0, smooth, variable=\x, blue, thick]
        plot ({\x}, {-0.8*\x});

    \draw[domain=0:2.5, smooth, variable=\x, blue, thick]
        plot ({\x}, {0.8*\x});
    \draw[thickline]
        (-2.5,2.25) -- (-2,1.74) -- (-1,0.85) -- (0,0)
        -- (1,0.85) -- (2,1.74) -- (2.5,2.25);

    \draw[dashed, line width=1.2pt] (-2,0) -- (-2,1.68);
    \draw[dashed, line width=1.2pt] (-1,0) -- (-1,0.82);
    \draw[dashed, line width=1.2pt] (2,0) -- (2,1.68);
    \draw[dashed, line width=1.2pt] (1,0) -- (1,0.82);

    \filldraw [black] (0,0) circle (1.2pt);
    \filldraw [black] (-2,1.74) circle (1.2pt);
    \filldraw [black] (2,1.74) circle (1.2pt);
    \filldraw [black] (-1,0.85) circle (1.2pt);
    \filldraw [black] (1,0.85) circle (1.2pt);
    
    \node at (-2,-0.3) {\large $-q_2$};
    \node at (-1,-0.3) {\large $-q_1$};
    \node at (0,-0.3) {\large $0$};
    \node at (1,-0.3) {\large $q_1$};
    \node at (2,-0.3) {\large $q_2$};

\end{tikzpicture}
         \caption{PAR as $|x|$}
     \end{subfigure}
     \hfill
     \begin{subfigure}[b]{0.28\textwidth}
         \centering
         \begin{tikzpicture}[scale=0.8]
    \tikzset{myarrow/.style={-{Straight Barb[scale=1]}, line width=\mylinewidthone}}
    \tikzset{thickline/.style={line width=1.2pt}}
    
    \draw[myarrow] (-3,0) -- (3,0);
    \draw[myarrow] (0,0) -- (0,3);

    \draw[domain=-2.5:0, smooth, variable=\x, blue, thick]
        plot ({\x}, {1.3*sqrt(-\x)});

    \draw[domain=0:2.5, smooth, variable=\x, blue, thick]
        plot ({\x}, {1.3*sqrt(\x)});
    \draw[thickline]
        (-2.5,2.0554) -- (-2,1.838) -- (-1,1.3) -- (0,0)
        -- (1,1.3) -- (2,1.838) -- (2.5,2.0554);

    \draw[dashed, line width=1.2pt] (-2,0) -- (-2,1.838);
    \draw[dashed, line width=1.2pt] (-1,0) -- (-1,1.3);
    \draw[dashed, line width=1.2pt] (2,0) -- (2,1.838);
    \draw[dashed, line width=1.2pt] (1,0) -- (1,1.3);

    \filldraw [black] (0,0) circle (1.2pt);
    \filldraw [black] (-2,1.838) circle (1.2pt);
    \filldraw [black] (2,1.838) circle (1.2pt);
    \filldraw [black] (-1,1.3) circle (1.2pt);
    \filldraw [black] (1,1.3) circle (1.2pt);
    
    \node at (-2,-0.3) {\large $-q_2$};
    \node at (-1,-0.3) {\large $-q_1$};
    \node at (0,-0.3) {\large $0$};
    \node at (1,-0.3) {\large $q_1$};
    \node at (2,-0.3) {\large $q_2$};

\end{tikzpicture}
         \caption{PAR as $\sqrt{|x|}$}
     \end{subfigure}
     \caption{PARs as approximations to classical regularizers. The PARs in black color approximate the standard regularizers shown in blue.}
    \label{fig::par-as-different-regularizers}
\end{figure}
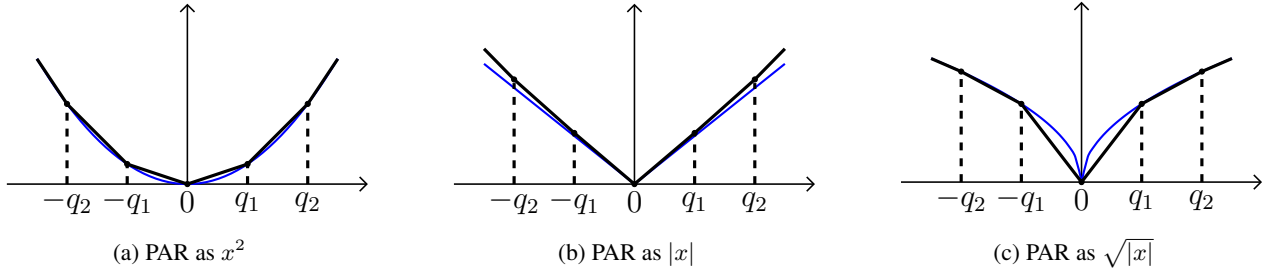
In this section, we demonstrate that a special class of PARs can effectively approximate ridge regression
\begin{equation}
    F_{\ridge}(\vx)=\frac{1}{2n}\norm{\mA\vx-\vb}^2+\frac{\lambda}{2}\norm{\vx}^2,
\end{equation}
which admits a closed-form solution 
$$\vx^{\star}_{\ridge}=\argmin F_{\ridge}(\vx)=(\mA^{\top}\mA+n\lambda\mI)^{\top}\mA^{\top}\vb.
$$
\paragraph{PAR formulation.} We consider a special class of PARs, illustrated in \Cref{fig::PAR-as-l2}. For this class, the quantization set is $\cQ=\{0, \pm q, \pm 2q, \cdots\}$, and the slope set is $\cA=\{q, 2q, \cdots\}$ where $q$ is the distance between adjacent quantization levels.\footnote{We can also design PARs with nonuniform quantization intervals to approximate the $\ell_2$-regularizer.}

\paragraph{Data model.} We consider the fixed design regime, where the design matrix $\mA$ is fixed. The response vector $\vb$ is generated by $\vb=\mA\vx^{\star}+\vepsilon$, with $\vepsilon=[\epsilon_1, \ldots, \epsilon_n]^{\top}$ where each $\epsilon_i$ is an independent random variable with zero mean and variance $\sigma^2$.\footnote{It may be possible to extend these results to the random design regime by incorporating results from \cite{hsu2014random}.} 

We denote the sample covariance matrix by $\widehat{\mSigma}:=\frac{1}{n}\mA^{\top}\mA$. For any estimator $\vx$, we define the in-sample risk and excess risk as 
\begin{equation}
    \cR(\vx):=\bE_{\Tilde{\vb}\sim \cD}\left[\frac{1}{n}\norm{\mA\vx-\Tilde{\vb}}^2\right], \quad \cE(\vx)=\cR(\vx)-\cR^{\star}=\norm{\vx-\vx^{\star}}^2_{\widehat{\mSigma}}.
\end{equation}
Let $\cR^{\star}=\min_{\vx\in \bR^d}\cR(\vx)$ represents the optimal risk.
Our next theorem characterizes the distances between the PAR and the ridge solutions.

\begin{theorem}
\label{thm::ridge-approximation}
    The distances between $\vx^{\star}_{\PAR}$ and $\vx^{\star}_{\ridge}$ are characterized as follows
    \begin{equation}
        \norm{\vx^{\star}_{\PAR}-\vx^{\star}_{\ridge}}\leq \sqrt{\frac{d}{2}}q, \quad \text{and}\quad\norm{\vx^{\star}_{\PAR}-\vx^{\star}_{\ridge}}_{\widehat{\mSigma}}\leq \sqrt{\frac{d\lambda}{2}}q.
    \end{equation}
\end{theorem}
The above two bounds both scale with the quantization level $q$. However, unlike the distance evaluated in the $\ell_2$-norm, the one in the Mahalanobis norm $\norm{\cdot}_{\widehat{\mSigma}}$ depends on the regularization strength $\lambda$. Specifically, as the regularization strength $\lambda$ decreases, the two estimators become closer and closer in the Mahalanobis norm $\norm{\cdot}_{\widehat{\mSigma}}$. This discrepancy arises because the least-square loss $f(\vx)=\frac{1}{2n}\norm{\mA\vx-\vb}^2$ is not strongly convex with respect to the $\ell_2$-norm when $n\ll d$, whereas it remains strongly convex in the Mahalanobis norm. This result is particularly appealing, as our next theorem establishes that the PAR-regularized solution enjoys statistical guarantees comparable to the ridge estimator, provided that the quantization gap $q$ is sufficiently small.

\begin{proof}
    To start with, we first show that the two loss functions $F_{\PAR}(\vx)$ and $F_{\ridge}(\vx)$ are uniformly close to each other in the sense that
    \begin{equation}
        \sup_{\vx\in \bR^d}\{F_{\PAR}(\vx)-F_{\ridge}(\vx)\}\leq \frac{d\lambda q^2}{8}.
    \end{equation}
    To this end, we first consider the 1-dimensional function $\Psi(x)-\frac{1}{2}x^2$ for $x\in \bR$. Suppose that $x\in [kq, (k+1)q)$ for some $k\in \bZ$. Without loss of generality, we assume $k\geq 0$. Then, we have 
    \begin{equation}
        0\leq \Psi(x)-\frac{1}{2}x^2=\left(k+\frac{1}{2}\right)q(x-kq)+\frac{k^2q^2}{2}-\frac{1}{2}x^2.
    \end{equation}
    It attains its maximum at the point $x^{\star}=\left(k+\frac{1}{2}\right)q$ with the optimal value $\frac{1}{8}q^2$. Then, the argument follows by taking summation over all the coordinates.
    
    Provided this result, we can establish the following relationship between the losses $F_{\ridge}(\vx^{\star}_{\PAR})$ and $F_{\ridge}(\vx^{\star}_{\ridge})$:
    \begin{equation}
    \label{eq::ridge-PAR}
        \begin{aligned}
            F_{\ridge}(\vx^{\star}_{\PAR})\leq F_{\PAR}(\vx^{\star}_{\PAR})+\frac{d\lambda q^2}{8}\leq F_{\PAR}(\vx^{\star}_{\ridge})+\frac{d\lambda q^2}{8}\leq F_{\ridge}(\vx^{\star}_{\ridge})+\frac{d\lambda q^2}{4}.
        \end{aligned}
    \end{equation}
    Now, we first consider the distance in the $\ell_2$-norm. Note that the regularizer $\frac{\lambda}{2}x^2$ is $\lambda$-strongly convex with respect to the $\ell_2$-norm and the loss function $\frac{1}{2n}\norm{\mA\vx-\vb}^2$ is convex. Hence, we have 
    \begin{equation}
        \begin{aligned}
            F_{\ridge}(\vx^{\star}_{\PAR})&\geq F_{\ridge}(\vx^{\star}_{\ridge})+\inner{\nabla F_{\ridge}(\vx^{\star}_{\ridge})}{\vx^{\star}_{\PAR}-\vx^{\star}_{\ridge}}+\frac{\lambda}{2}\norm{\vx^{\star}_{\PAR}-\vx^{\star}_{\ridge}}^2\\
            &= F_{\ridge}(\vx^{\star}_{\ridge})+\frac{\lambda}{2}\norm{\vx^{\star}_{\PAR}-\vx^{\star}_{\ridge}}^2.
        \end{aligned}
    \end{equation}
    Here we use the optimality condition that $\nabla F_{\ridge}(\vx^{\star}_{\ridge})=\vzero$.
    Substituting this inequality into \Cref{eq::ridge-PAR}, we derive the desired result
    \begin{equation}
        \norm{\vx^{\star}_{\PAR}-\vx^{\star}_{\ridge}}\leq \sqrt{\frac{d}{2}} q.
    \end{equation}
    
    To control the Mahalanobis norm, we note that the least-square loss $f(\vx)=\frac{1}{2n}\norm{\mA\vx-\vb}^2$ is $1$-strongly convex with respect to the norm $\norm{\cdot}_{\widehat{\mSigma}}$. To show this, it suffices to prove that for any $t\in [0, 1]$ and any $\vx, \vy\in \bR^d$, the following inequality holds 
    \begin{equation}
        f(t\vx+(1-t)\vy)\leq tf(\vx)+(1-t)f(\vy)-\frac{t(1-t)}{2}\norm{\vx-\vy}^2_{\widehat{\mSigma}}.
    \end{equation}
    This follows by the following equality
    \begin{equation}
        \begin{aligned}
            f(t\vx+(1-t)\vy)&=\frac{1}{2n}\norm{t\left(\mA\vx-\vb\right)+(1-t)\left(\mA\vy-\vb\right)}^2\\
            &=tf(\vx)+(1-t)f(\vy)-\frac{t(1-t)}{2n}\norm{\mA\left(\vx-\vy\right)}^2\\
            &=tf(\vx)+(1-t)f(\vy)-\frac{t(1-t)}{2}\norm{\vx-\vy}^2_{\widehat{\mSigma}}.
        \end{aligned}
    \end{equation}
    Therefore, leveraging the property of the strong convexity, we have
    \begin{equation}
        \begin{aligned}
            F_{\ridge}(\vx^{\star}_{\PAR})&\geq F_{\ridge}(\vx^{\star}_{\ridge})+\inner{\nabla F_{\ridge}(\vx^{\star}_{\ridge})}{\vx^{\star}_{\PAR}-\vx^{\star}_{\ridge}}+\frac{1}{2}\norm{\vx^{\star}_{\PAR}-\vx^{\star}_{\ridge}}^2_{\widehat{\mSigma}}\\
            &= F_{\ridge}(\vx^{\star}_{\ridge})+\frac{1}{2}\norm{\vx^{\star}_{\PAR}-\vx^{\star}_{\ridge}}^2_{\widehat{\mSigma}}.
        \end{aligned}
    \end{equation}
    Substituting this inequality into \Cref{eq::ridge-PAR}, we derive that 
    \begin{equation}
        \norm{\vx^{\star}_{\PAR}-\vx^{\star}_{\ridge}}_{\widehat{\mSigma}}^2\leq \frac{d\lambda q^2}{2}.
    \end{equation}
    This completes the proof.
\end{proof}
Our next theorem characterizes the excess risk of the PAR-regularized solution.
\begin{theorem}[Excess risk]
\label{thm::par-as-ridge}
    Consider the data model and PAR formulation described above. If the regularization strength in (\ref{eq::linear-regression}) is chosen as $\lambda=\Theta\left(\frac{\sigma}{\norm{\vx^{\star}}+\sqrt{d} q}\sqrt{\frac{\tr(\widehat{\mSigma})}{n}}\right)$, the excess risk of the PAR-regularized solution $\vx^{\star}_{\PAR}$ satisfies
    \begin{equation}
        \cE(\vx^{\star}_{\PAR})\lesssim \sigma\left(\norm{\vx^{\star}}+\sqrt{d} q\right)\sqrt{\frac{\tr(\widehat{\mSigma})}{n}}.
    \end{equation}
\end{theorem}
\paragraph{Statistical guarantee.} If the quantization level is chosen as $ q\leq \frac{\norm{\vx^{\star}}}{\sqrt{d}}$, the excess risk of the PAR-regularized solution is $O\Big(\sigma\norm{\vx^{\star}}\sqrt{\tr(\widehat{\mSigma})/n}\Big)$, which is the same as that of the ridge regression estimator. 

\paragraph{Quantization guarantee.} \Cref{thm::main} implies that at least $d-n$ coordinates of $\vx^{\star}_{\PAR}$ are quantized. Moreover, if the covariate spectrum is concentrated in only a few directions, which means that the effective rank $\tr(\widehat{\mSigma})/\|\widehat{\mSigma}\|\ll d$, then the required sample complexity can be significantly smaller than the dimensionality, i.e., $n\ll d$. In this scenario, most coordinates of $\vx^{\star}_{\PAR}$ are quantized.

\paragraph{Comparison with simple estimators.} One might argue that the quantization result derived in \Cref{thm::main} is too weak. For instance, we can easily construct an estimator with the same quantization guarantee: we simply set $d-n$ coordinates to be zero and pick the remaining $n$ coordinates by solving a linear equation $\mA_{:n}\vx_{:n}=\vb$. While this approach achieves zero training loss, it fails to generalize well, highlighting the advantage of the PAR estimator, which maintains both quantization efficiency and strong statistical guarantees.

\paragraph{Storage advantage.} The PAR estimator offers significant storage benefits compared to the ridge estimator. To quantify this, we first consider the ridge estimator, whose parameters are typically dense and stored in single-precision floating-point format (FP32), requiring $32$ bits to store each parameter. In contrast, the PAR estimator reduces storage via quantization. To match the statistical accuracy of ridge regression, it suffices to select the quantization gap $ q \leq \frac{\norm{\vx^{\star}}}{\sqrt{d}}$. A crude bound on the maximum entry of the PAR estimator gives
\begin{equation}
    \begin{aligned}
        \norm{\vx_{\PAR}^{\star}}_{\infty} \leq \norm{\vx^{\star}_{\PAR}-\vx^{\star}_{\ridge}}+\norm{\vx^{\star}_{\ridge}}\leq \sqrt{\frac{d}{2}} q+\norm{\vx^{\star}}\leq 2\norm{\vx^{\star}}.
    \end{aligned}
\end{equation}
Therefore, storing the magnitude of each parameter requires at most $\lceil\ln(2\sqrt{d})\rceil$ bits, plus one additional bit for the sign. For dimensions up to $1 \times 10^8$, this amounts to roughly $16$ bits per parameter, resulting in a $2\times$ reduction in storage compared to the ridge estimator.

Before proving \Cref{thm::par-as-ridge}, we first introduce the following classic textbook result on Ridge regression.
\begin{proposition}[Proposition 3.7 in \cite{bach2024learning}] 
\label{prop::ridge}
For ridge estimator $\vx^{\star}_{\ridge}$, its excess risk is characterized as
    \begin{equation}
\cE(\vx^{\star}_{\ridge})\leq \frac{\lambda}{2}\norm{\vx^{\star}}^2+\frac{\sigma^2\tr(\widehat{\mSigma})}{2\lambda n}.
\end{equation}
\end{proposition} 
We now proceed to prove \Cref{thm::par-as-ridge}.
\begin{proof}[Proof of \Cref{thm::par-as-ridge}]
First, we have the following decomposition 
\begin{equation}
    \cE(\vx^{\star}_{\PAR})=\norm{\vx^{\star}_{\PAR}-\vx^{\star}}_{\widehat{\mSigma}}^2\leq 2\norm{\vx^{\star}_{\PAR}-\vx^{\star}_{\ridge}}_{\widehat{\mSigma}}^2+2\norm{\vx^{\star}_{\ridge}-\vx^{\star}}_{\widehat{\mSigma}}^2.
\end{equation}
By \Cref{thm::ridge-approximation}, the first term can be bounded as
\begin{equation}
        \norm{\vx^{\star}_{\PAR}-\vx^{\star}_{\ridge}}_{\widehat{\mSigma}}\leq \sqrt{\frac{d\lambda}{2}} q.
    \end{equation}
    For the second term, applying \Cref{prop::ridge} yields
    \begin{equation}
        \cE(\vx^{\star}_{\ridge})=\norm{\vx^{\star}_{\ridge}-\vx^{\star}}_{\widehat{\mSigma}}^2\leq \frac{\lambda}{2}\norm{\vx^{\star}}^2+\frac{\sigma^2\tr(\widehat{\mSigma})}{2\lambda n}.
    \end{equation}
    Combining the above two recipes, we obtain that 
    \begin{equation}
        \cE(\vx^{\star}_{\PAR})\leq \lambda \left(d q^2+\norm{\vx^{\star}}^2\right)+\frac{\sigma^2\tr(\widehat{\mSigma})}{\lambda n}.
    \end{equation}
    Therefore, upon setting $\lambda=\Theta\left(\frac{\sigma}{\norm{\vx^{\star}}+\sqrt{d} q}\sqrt{\frac{\tr(\widehat{\mSigma})}{n}}\right)$, we derive that
    \begin{equation}
        \cE(\vx^{\star}_{\PAR})\lesssim \sigma\left(\norm{\vx^{\star}}+\sqrt{d}q\right)\sqrt{\frac{\tr(\widehat{\mSigma})}{n}}.
    \end{equation}
    This completes the proof.
\end{proof}
\subsection{PAR as Lasso and nonconvex regularizers}
In this section, we demonstrate that a special class of PAR regularizers can effectively approximate the Lasso objective
\begin{equation}
    F_{\Lasso}(\vx)=\frac{1}{2n}\norm{\mA\vx-\vb}^2+\frac{\lambda}{2}\norm{\vx}_1.
\end{equation}
as well as the nonconvex regularizers, including several commonly used in sparse linear regression, such as the bridge regularizer ($\Psi(\vx) = \norm{\vx}_p^p$ for $0 < p < 1$),  smoothly clipped
absolute deviations (SCAD) penalty \cite{fan2001variable}, and minimax concave penalty (MCP) \cite{zhang2010nearly}. 
Nonconvex regularizers are introduced to better approximate the $\ell_0$-norm and to address the bias issue of the Lasso, which penalizes large coefficients more heavily. See \cite{zhang2012general} for a comprehensive survey of nonconvex regularizers.

While the $\ell_1$-regularizer $\Psi(\vx)=\norm{\vx}_1$ is itself a special case of a PAR, it primarily encourages sparsity by promoting solutions concentrated at zero. In contrast, we propose a richer class of PARs that not only approximate the $\ell_1$-penalty but also introduce additional quantization levels beyond zero. This structure enables the regularizer to promote parameter values near multiple predefined levels, thereby facilitating quantization while retaining statistical properties comparable to those of the standard Lasso solution.

\paragraph{Data model.} We consider the following sparse linear regression setting:
\begin{assumption}
    The true solution $\vx^{\star}$ is $s$-sparse, i.e., its support $S=\supp(\vx^{\star})$ satisfies $|S|=s$.
\end{assumption}
\begin{assumption}[Restricted eigenvalue]
    We assume the design matrix $\mA$ satisfies the restricted eigenvalue (RE) condition over $S=\mathrm{supp}(\vx^{\star})$ with parameters $(\alpha, \gamma)$, that is
    \begin{equation}
        \frac{1}{n}\norm{\mA\vv}^2\geq \gamma \norm{\vv}^2, \quad \forall \vv\in \cC_{\alpha}(S):=\{\vv\in \bR^d:\norm{\vv_{S^c}}_1\leq \alpha\norm{\vv_S}_1\}.
    \end{equation}
\end{assumption}
This condition holds for a broad class of random design matrices, particularly those with sub-Gaussian or isotropic rows, provided that the sample size is sufficiently large relative to the sparsity level $s$. In particular, for a Gaussian design matrix, where $\mA \in \bR^{n\times d}$ has i.i.d. $\cN (\vzero,\mSigma)$ rows, the restricted eigenvalue condition with parameters $(\alpha, \gamma)$ holds with probability at least $1-\exp\{-\Omega(n)\}$ as long as the sample size $n\gtrsim \frac{\norm{\mSigma}^2(1+\alpha)^2}{\gamma}s\log(d)$ \cite[Corollary~1]{raskutti2010restricted}.

\paragraph{PAR formulation.}
We consider a PAR $\Psi(\cdot)$ with quantization values $\cQ=\{0, \pm q_1, \pm q_2, \cdots\}$ where $0<q_1<q_2<\cdots$ and slopes $\cA=\{\pm a_1, \pm a_2, \cdots\}$. Let $a_{\max} = \max\{a \in \cA\}$ denote the maximum slope magnitude. We assume that $\Psi$ satisfies a linear growth condition, i.e., there exists a universal constant $k > 0$ such that $\Psi(x) \geq \nu|x|$ for all $x \in \bR$.

The following two examples illustrate this class of PARs.
\begin{example}[Convex PAR]
\label{example::1}
    Consider a convex PAR $\Psi(\cdot)$ with arbitrary quantization values $\cQ$ and slopes $\cA=\{a_0, a_1, \cdots, a_m\}$ with $0<a_0<a_1<\cdots < a_m<\infty$. In this case, $a_{\max}=a_m$ and $\nu=a_0$ since $\Psi(x)\geq a_0 x$ for all $x\in \bR$.
\end{example}

\begin{example}[Quasiconvex PAR]
\label{example::2}
    Consider the quasiconvex PAR in \Cref{fig::quasiconvex-2}, which is characterized by
    \begin{equation}
    \Psi(x)=\begin{cases}
        |x|-\frac{k}{2}q\quad \text{if}\quad kq\leq |x|\leq \frac{2k+1}{2}q,\\
        \frac{k+1}{2}q\quad \text{if}\quad\frac{2k+1}{2}q\leq |x|\leq (k+1)q,
    \end{cases}
\end{equation}
for integer $k \geq 0$ and fixed step size $q > 0$. It is straightforward to verify that this function is quasiconvex, with $a_{\max} = 1$ and $\nu = \frac{1}{2}$.
\end{example}

We now characterize the statistical guarantees for the PAR-regularized solution.
\begin{theorem}
\label{thm::PAR-as-Lasso}
    Consider the data model and PAR formulation described above. Suppose that the regularization strength satisfies $\lambda\geq \frac{\norm{\mA^{\top}\vepsilon}_{\infty}}{2\nu n}$, and the design matrix $\mA$ satisfies restricted eigenvalue condition with parameters $\left(\frac{3a_{\max}}{\nu},\gamma\right)$. Then, the estimation error of the PAR-regularized solution is bounded as
    \begin{equation}
        \norm{\vx^{\star}_{\PAR}-\vx^{\star}}\leq \frac{3\lambda a_{\max}\sqrt{s}}{\gamma}.
    \end{equation}
\end{theorem}
Similar guarantees can be derived for the prediction error $\norm{\mA(\vx^{\star}_{\PAR}-\vx^{\star})}$ and the estimation error in $\ell_{\infty}$-norm $\norm{\vx^{\star}_{\PAR}-\vx^{\star}}_{\infty}$. However, we omit these results here and leave them for future work.
This result closely resembles the classic error bound for Lasso regression \cite[Theorem~11.1]{hastie2015statistical}.
In particular, if $a_{\max}\asymp k$, then the guarantees in \Cref{thm::PAR-as-Lasso} match those of Lasso up to constant factors. To illustrate this result, we consider the classical linear Gaussian model.
\begin{corollary}
    Suppose $a_{\max}=\Theta(\nu)$. Assume the design matrix $\mA$ has i.i.d. standard Gaussian entries and the noise vector $\vepsilon \sim \cN(\vzero, \sigma^2 \mI)$ is also i.i.d. Gaussian. If the sample size satisfies $n\gtrsim s\log(d)$, then with probability at least $1-\exp{-\Omega(\log(d))}$, the estimation error is bounded by 
    \begin{equation}
        \norm{\vx^{\star}_{\PAR}-\vx^{\star}}\lesssim \sigma\sqrt{\frac{s\log(d)}{n}}.
    \end{equation}
\end{corollary}

\begin{proof}[Proof of \Cref{thm::PAR-as-Lasso}]
    Note that $\vx^{\star}_{\PAR}$ is the minimizer of $F_{\PAR}(\vx)$. Hence, we have $F_{\PAR}(\vx^{\star}_{\PAR})\leq F_{\PAR}(\vx^{\star})$, that is,
    \begin{equation}
        \frac{1}{2n}\norm{\mA\vx^{\star}_{\PAR}-\vb}^2+\lambda\Psi(\vx^{\star}_{\PAR})\leq \frac{1}{2n}\norm{\vepsilon}^2+\lambda\Psi(\vx^{\star}).
    \end{equation}
    Rearranging this inequality and denoting $\vv=\vx^{\star}_{\PAR}-\vx^{\star}$ yields 
    \begin{equation}
        \begin{aligned}
            \frac{1}{2n}\norm{\mA\vv}^2&\leq \frac{1}{n}\inner{\mA^{\top}\vepsilon}{\vv}+\lambda\left(\Psi(\vx^{\star})-\Psi(\vx^{\star}_{\PAR})\right)\\
            &\stackrel{(a)}{\leq} \frac{\norm{\mA^{\top}\vepsilon}_{\infty}}{n}\norm{\vv}_1+\lambda\left(\Psi(\vx^{\star}_S)-\Psi(\vx^{\star}_{\PAR, S})\right)-\lambda\Psi\left(\vx^{\star}_{\PAR, S^c}\right)\\
            &\stackrel{(b)}{\leq} \frac{\norm{\mA^{\top}\vepsilon}_{\infty}}{n}\left(\norm{\vv_S}_1+\norm{\vv_{S^c}}_1\right)+\lambda a_{\max}\norm{\vv_S}_1-\lambda \nu\norm{\vv_{S^c}}_1\\
            &\stackrel{(c)}{\leq} \frac{3}{2}\lambda a_{\max}\norm{\vv_S}_1-\frac{1}{2}\lambda \nu\norm{\vv_{S^c}}_1.
        \end{aligned}
        \label{eq::64}
    \end{equation}
    Here in $(a)$, we use Hölder's inequality and the fact that $\Psi(x^\star_i)=0$ for all $i\in S^c$. In $(b)$, we use the triangle inequality and the facts that $\Psi(\cdot)$ is $a_{\max}$-Lipschitz and $\Psi(x)\geq \nu |x|$ for all $x$. In $(c)$, we use the condition that $\lambda\geq \frac{\norm{\mA^{\top}\vepsilon}_{\infty}}{2\nu n}$.
    To proceed, note that $\norm{\mA\vv}^2\geq 0$, which implies that $\norm{\vv_{S^c}}_1\leq \frac{3a_{\max}}{\nu}\norm{\vv_S}_1$.
    Therefore, we can apply the restricted eigenvalue condition, which yields
    \begin{equation}
        \gamma\norm{\vv}^2\stackrel{\text{RE condition}}{\leq} \frac{1}{n}\norm{\mA\vv}^2\stackrel{\text{\eqref{eq::64}}}{\leq} 3\lambda a_{\max}\norm{\vv_S}_1-\lambda \nu\norm{\vv_{S^c}}_1\leq 3\lambda a_{\max}\sqrt{s}\norm{\vv}.
    \end{equation}
    This accomplishes the proof.
\end{proof}
\section{Numerical experiments}
\label{sec::numerical-experiments}
In this section, we present numerical experiments to validate our theoretical results on quantization, optimization, and statistical performance. In \Cref{sec::exp-quantization-guarantee}, we demonstrate that the lower bound on the quantization rate is nearly tight in the linear regression setting. \Cref{sec::exp-optimization-alg} evaluates the performance of various optimization algorithms across different PAR formulations. Finally, \Cref{sec::exp-stats} examines the statistical performance of PAR-regularized models in both linear and logistic regression tasks. Our code is available at \url{https://github.com/jianhaoma/paro}.
\subsection{Validation of the quantization guarantee}  
\label{sec::exp-quantization-guarantee}
We first verify the tightness of the quantization rate lower bound given in \Cref{thm::main}, which states that the quantization rate is at least $1 - n/d$ where $n$ is the sample size and $d$ is the data dimension, and is independent of the regularization strength $\lambda$. To test this, we conduct extensive simulations on a linear regression task. The data dimension is set to $d = 200$, and the input matrix $\mA$ is generated with i.i.d. Gaussian entries. The ground truth $\vx^{\star}$ is randomly sampled, and the output is computed as $\vy = \mA \vx^{\star}$ without additional noise.  

\begin{sloppypar}
    In our experiments, we use a convex PAR with quantization values $\cQ=\bZ$ and slopes $\cA=\{\dots, -2, -1, 1, 2, \dots\}$. \Cref{fig::quantization-ratio-sample-size} reports the observed quantization rate across varying sample sizes and regularization strengths. The results closely match the theoretical lower bound and show that the quantization rate remains largely unaffected by $\lambda$, particularly when the sample size is small. This observation is further supported by \Cref{fig::quantization-ratio-regularization-strength}, where $\lambda$ is varied from $10^{-4}$ to $100$: the quantization rate consistently exceeds $0.9$, aligning with the theoretical bound of $1 - n/d = 1 - 20/200 = 0.9$.
\end{sloppypar}

Additionally, \Cref{fig::quantization-ratio-loss} shows that increasing $\lambda$ leads to higher training loss. Therefore, to balance quantization and model performance, we recommend using a relatively small regularization strength in practice.

\begin{figure}
     \centering
     \begin{subfigure}[b]{0.32\textwidth}
         \centering
         \includegraphics[width=\textwidth]{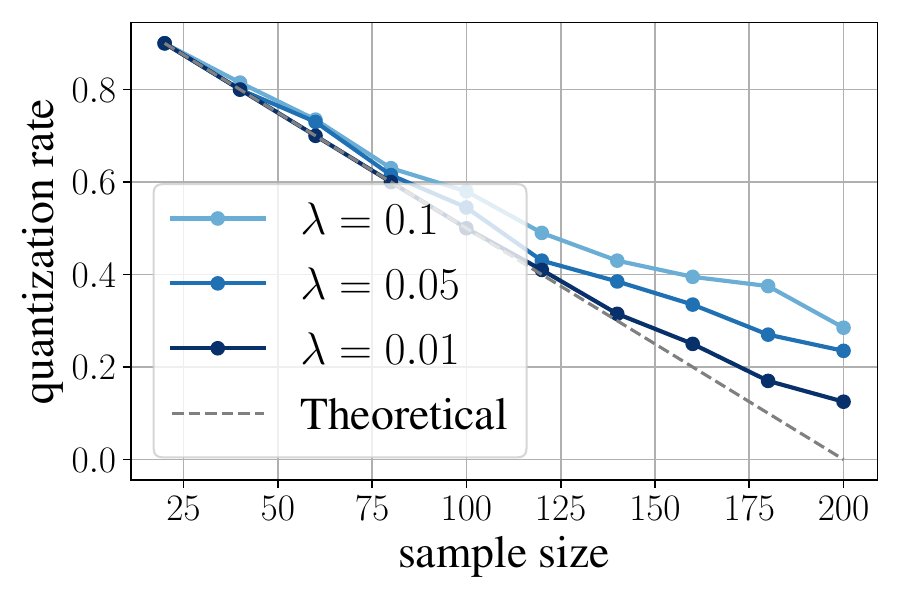}
         \caption{}
         \label{fig::quantization-ratio-sample-size}
     \end{subfigure}
     \hfill
     \begin{subfigure}[b]{0.32\textwidth}
         \centering
         \includegraphics[width=\textwidth]{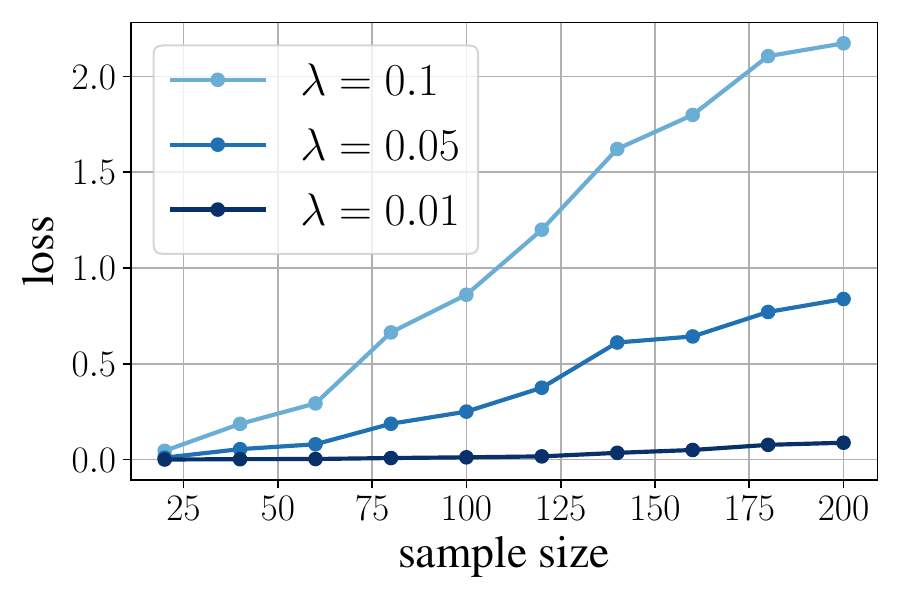}
         \caption{}
         \label{fig::quantization-ratio-loss}
     \end{subfigure}
     \hfill
     \begin{subfigure}[b]{0.32\textwidth}
         \centering
         \includegraphics[width=\textwidth]{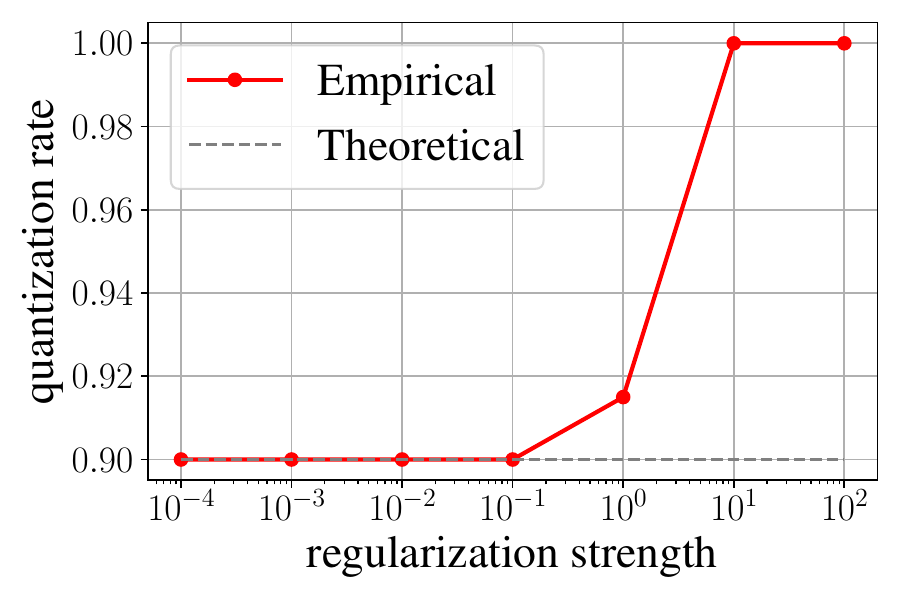}
         \caption{}
         \label{fig::quantization-ratio-regularization-strength}
     \end{subfigure}
        \caption{Empirical validation of the quantization guarantee on linear regression with $d=200$. Panel (a) tests the effect of sample size; panel (b) shows the impact of $\lambda$ on training loss; and panel (c) examines the robustness of the quantization rate to $\lambda$ for the case $n=20$ and $d=200$.}
        \label{fig::statistical-performances}
\end{figure}

\subsection{Comparison of different optimization algorithms and PAR variants}
\label{sec::exp-optimization-alg}
\paragraph{Convergence Behavior Across Optimization Algorithms.}
\begin{figure}
     \centering
     \begin{subfigure}[t]{0.32\textwidth}
         \centering
         \includegraphics[width=\textwidth]{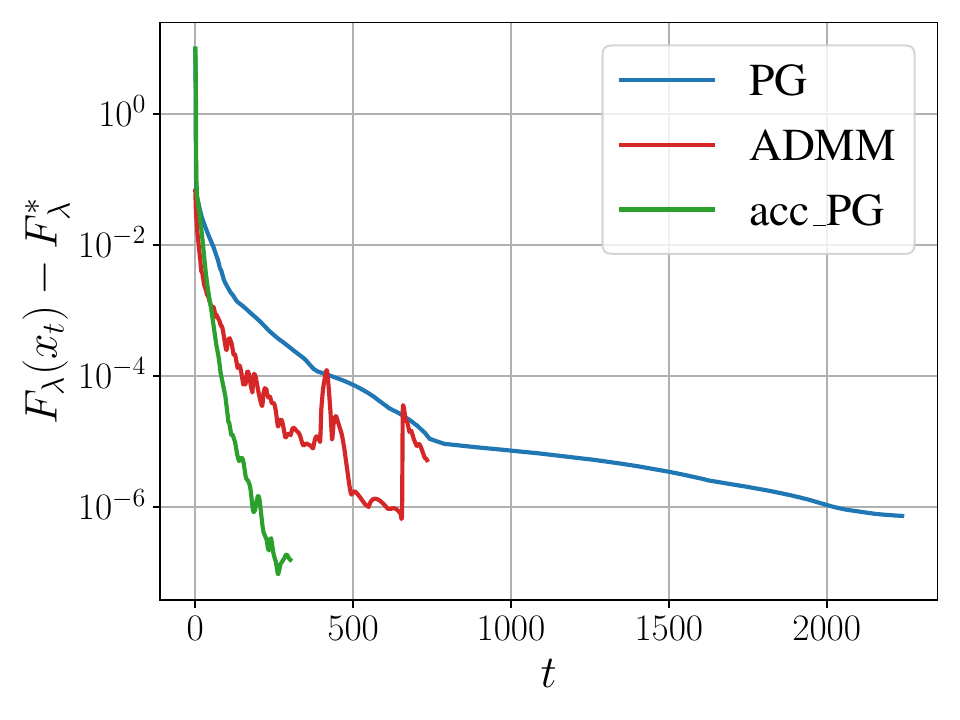}
         \label{fig:convex_different_algos}
     \end{subfigure}
     \hfill
     \begin{subfigure}[t]{0.32\textwidth}
         \centering
         \includegraphics[width=\textwidth]{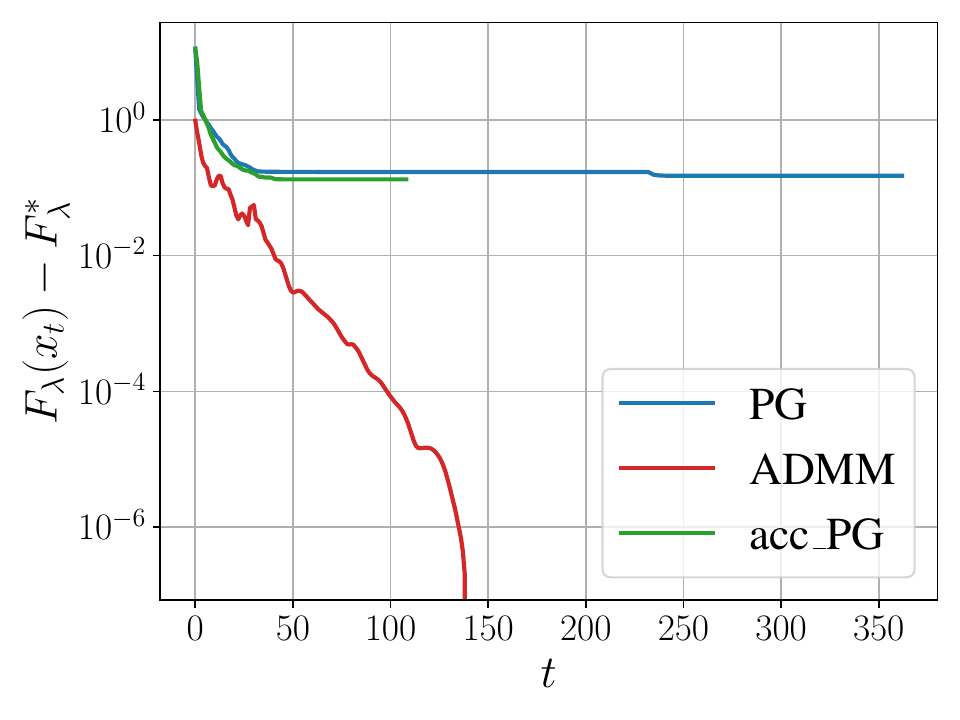}
     \end{subfigure}
     \hfill
     \begin{subfigure}[t]{0.32\textwidth}
         \centering
         \includegraphics[width=\textwidth]{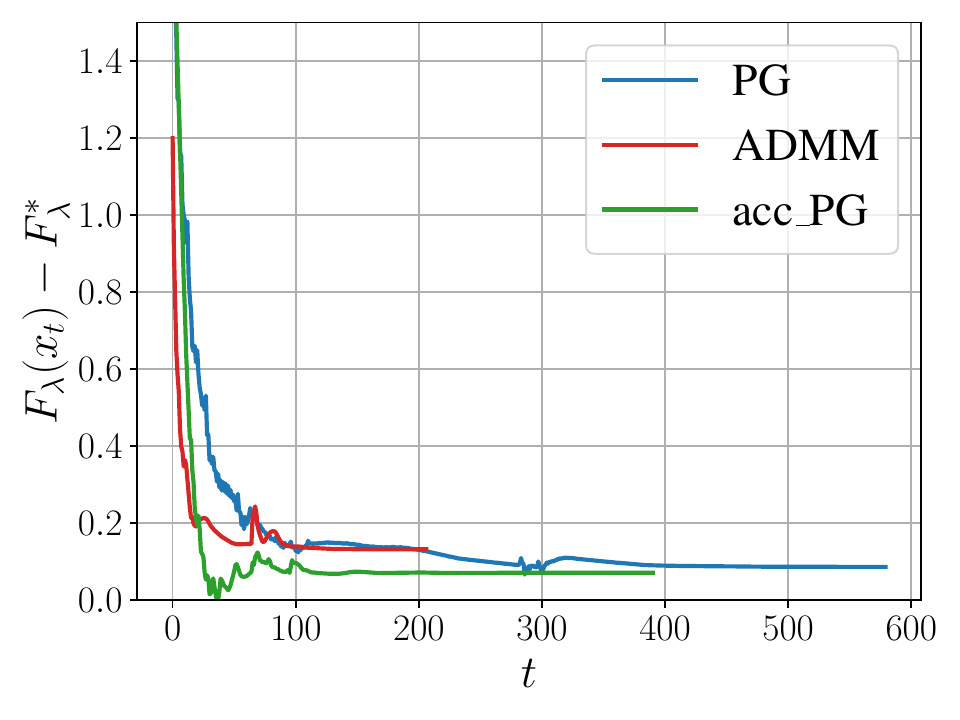}
         \label{fig:nonconvex_different_algos}
     \end{subfigure}\\
     \begin{subfigure}[b]{0.32\textwidth}
         \centering
         \includegraphics[width=\textwidth]{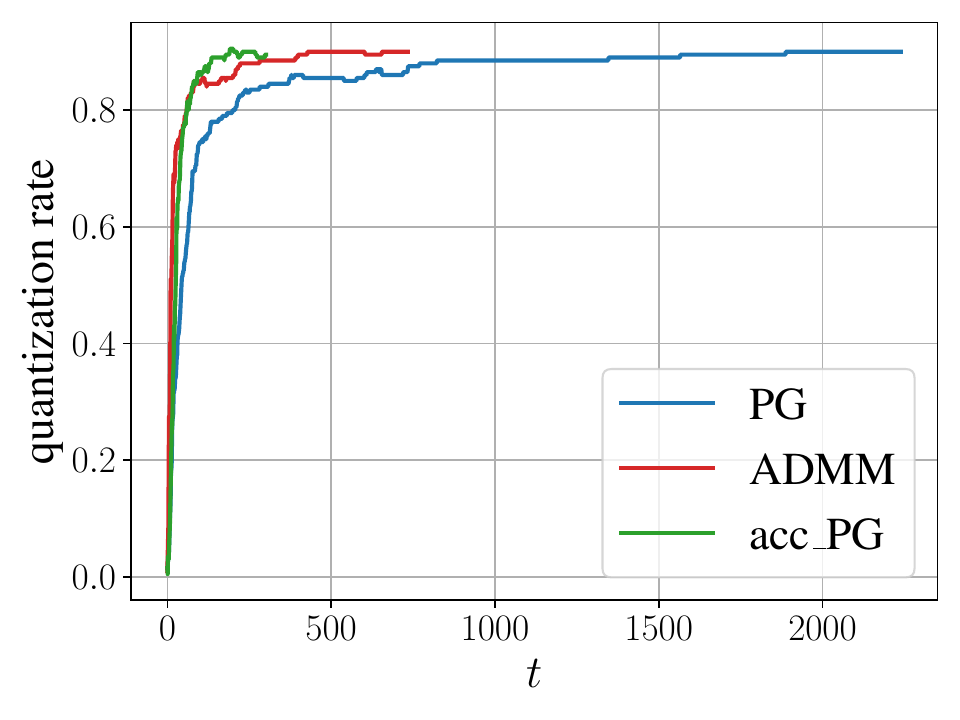}
         \caption{Convex PAR}
         \label{fig:convex_different_algos-quant-ratio}
     \end{subfigure}
     \hfill
     \begin{subfigure}[b]{0.32\textwidth}
         \centering
         \includegraphics[width=\textwidth]{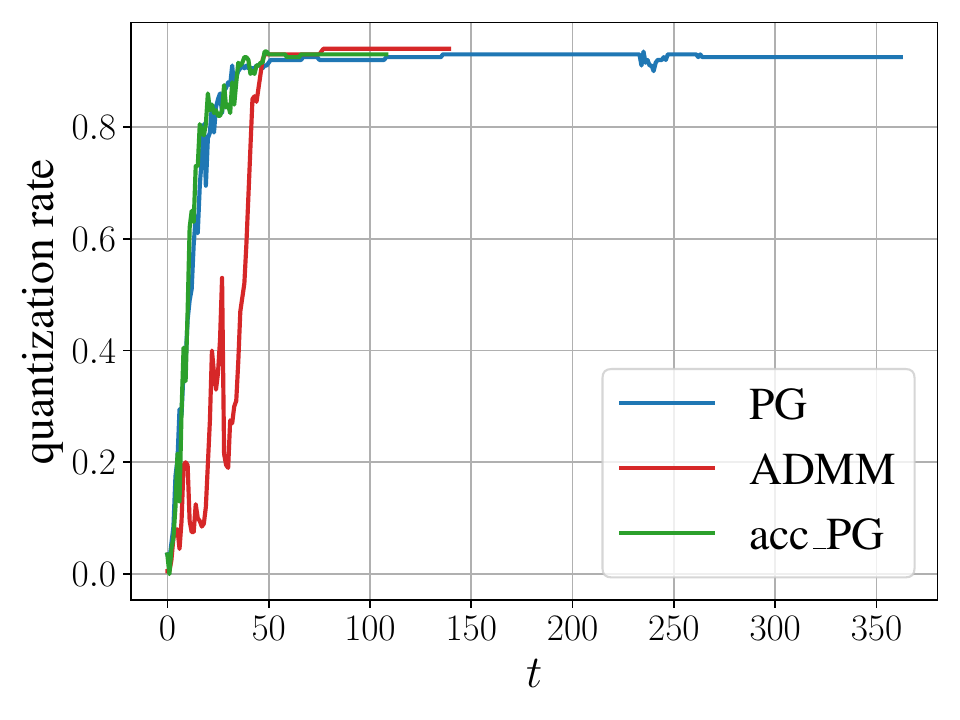}
         \caption{Quasiconvex PAR}
         \label{fig:quasiconvex_different_algos-quant-ratio}
     \end{subfigure}
     \hfill
     \begin{subfigure}[b]{0.32\textwidth}
         \centering
         \includegraphics[width=\textwidth]{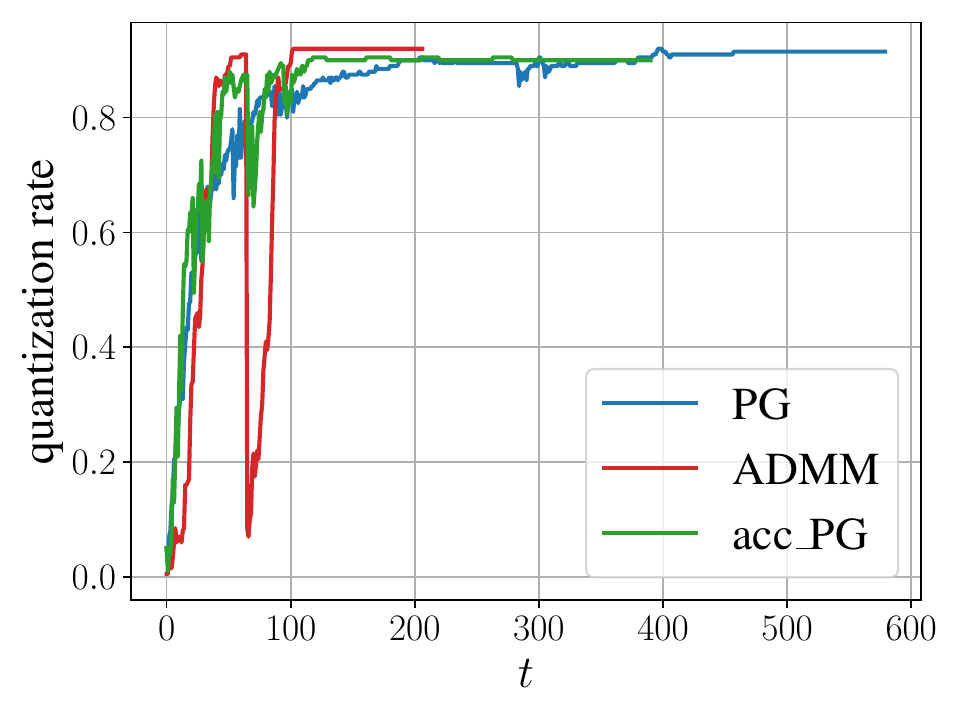}
         \caption{Nonconvex PAR}
         \label{fig:nonconvex_different_algos-quant-ratio}
     \end{subfigure}
        \caption{Comparison of optimization algorithms $\texttt{PG}$, $\texttt{acc\_PG}$, and $\texttt{ADMM}$ for linear regression with convex (left), quasiconvex (middle), and nonconvex (right) PARs. The problem dimension is $200$ and the sample size is $20$. All algorithms determine the step size using line search.}
        \label{fig::optimization-different-PARs}
\end{figure}
We evaluate the optimization performance of three algorithms: proximal gradient ($\texttt{PG}$), accelerated proximal gradient ($\texttt{acc\_PG}$), and ADMM ($\texttt{ADMM}$). The task is a linear regression problem regularized by three types of PARs: convex, quasiconvex, and nonconvex. The implementations of these three algorithms follow \Cref{sec::PG,sec::ADMM}. All methods incorporate a backtracking line search to select the step size.

We use synthetic data with feature dimension $d = 200$ and sample size $n = 20$. The true parameter $\vx^\star$ is generated randomly, and the design matrix $\mA$ is drawn from a standard Gaussian distribution. The response vector is generated as $\vb = \mA \vx^\star + \vepsilon$, where $\vepsilon \sim \cN(\bm{0}, 0.01 \mI)$.

\Cref{fig::optimization-different-PARs} summarizes the convergence behavior across different regularizers. For the convex PAR, all algorithms exhibit linear convergence, with $\texttt{acc\_PG}$ achieving the fastest rate, followed by $\texttt{ADMM}$ and then $\texttt{PG}$. For the quasiconvex PAR, $\texttt{ADMM}$ substantially outperforms both $\texttt{PG}$ and $\texttt{acc\_PG}$. In the case of the nonconvex PAR, all three algorithms show comparable convergence patterns. Interestingly, the optimal objective value is typically achieved early, prior to full convergence to a critical point, which might be attributed to the nonconvex nature of the regularizer.

\paragraph{Effect of PAR structure on quantization performance.}
\begin{figure}
\centering
\includegraphics[width=0.45\linewidth]{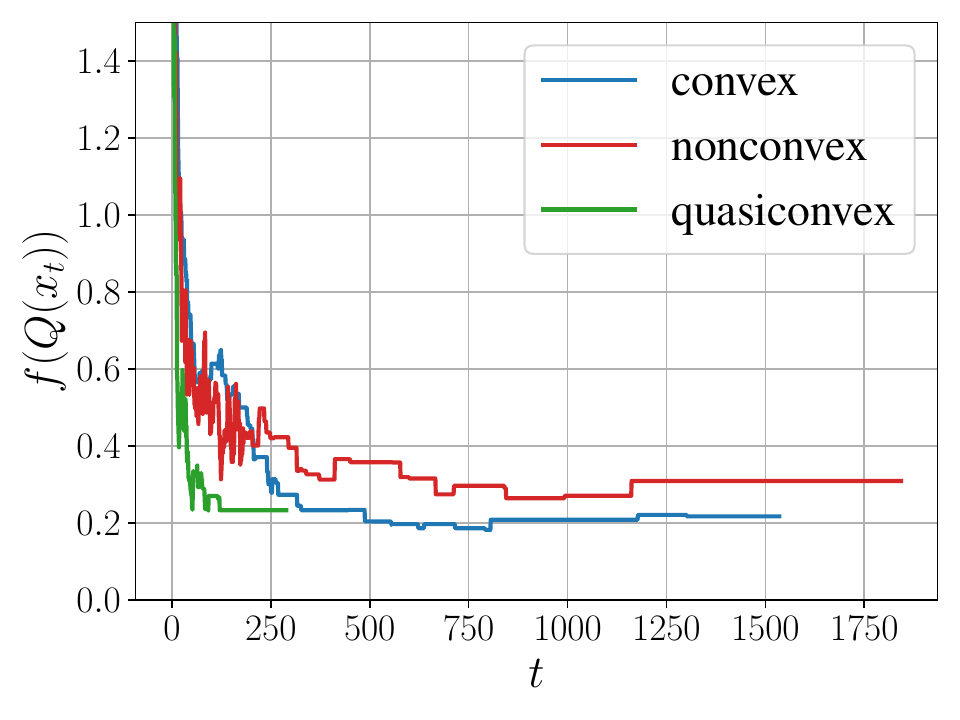}
\caption{Comparison of three different PARs on a linear regression task. Here dimension is $1000$ and the sample size is $100$. To ensure a fair comparison, we use the same set of quantization values $\cQ$ and evaluate the objective value of the fully quantized solutions. Specifically, at each iteration, the current solution $\vx^t$ is projected onto $\cQ$ to obtain $Q(\vx^t)$, and we report the unregularized objective value $f(Q(\vx^t))$. For each PAR, we report the best performance achieved over the regularization strengths $\{0.01, 0.015, 0.02, 0.05\}$.}
\label{fig:performance-different-pars}
\end{figure}
In this simulation, we examine how the structure of the PAR, whether convex, quasiconvex, or nonconvex, influences the quality of the final solution. We generate synthetic data with $d = 1000$ features and $n = 100$ samples. To ensure a fair comparison, all methods use the same quantization set $\cQ$. We adopt $\texttt{ADMM}$ as the optimization algorithm, as it consistently performs slightly better than $\texttt{PG}$ and $\texttt{acc\_PG}$ in this setting; however, the choice of solver does not significantly impact the observed trends. At each iteration, the current iterate $\vx^t$ is projected onto $\cQ$ to obtain a fully quantized solution $Q(\vx^t)$, and we evaluate the unregularized objective $f(Q(\vx^t))$ to measure the solution quality. For each regularizer, we report the best performance across regularization strengths $\{0.01, 0.015, 0.02, 0.05\}$.

As shown in \Cref{fig:performance-different-pars}, the convex PAR outperforms both quasiconvex and nonconvex counterparts, with the quasiconvex variant performing slightly better than the nonconvex one. This result highlights the potential benefits of convexity in guiding the algorithm toward high-quality solutions.

\begin{figure}[t]
     \centering
     \begin{subfigure}[b]{0.3\textwidth}
         \centering
         \includegraphics[width=\textwidth]{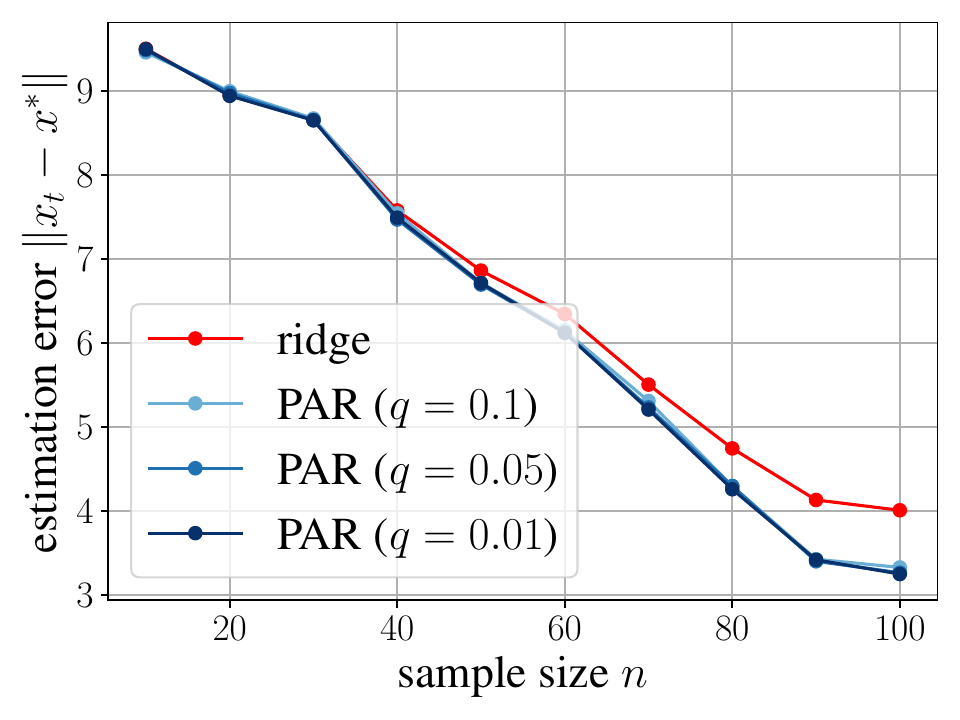}
         \caption{Ridge regularizer}
         
     \end{subfigure}
     \hfill
     \begin{subfigure}[b]{0.3\textwidth}
         \centering
         \includegraphics[width=\textwidth]{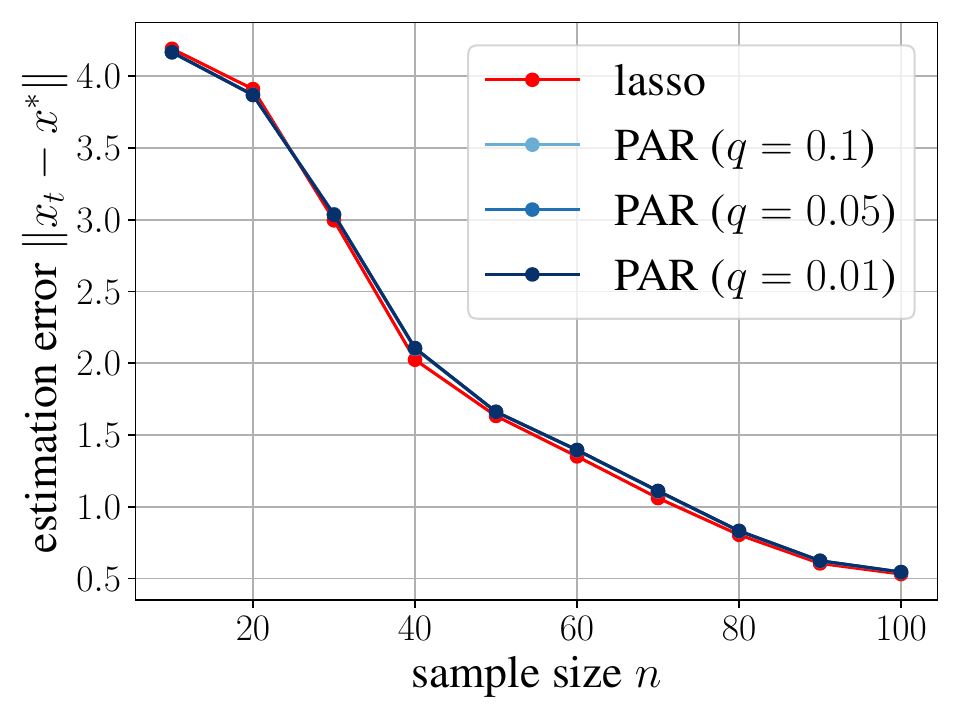}
         \caption{$\ell_1$-regularizer}
         
     \end{subfigure}
     \hfill
     \begin{subfigure}[b]{0.3\textwidth}
         \centering
         \includegraphics[width=\textwidth]{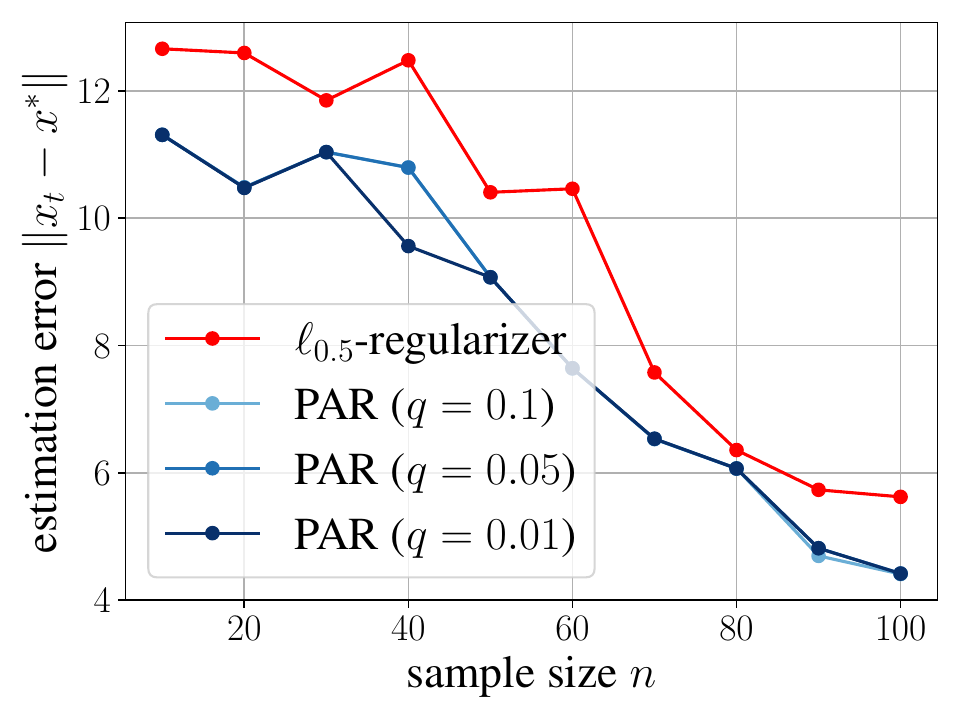}
         \caption{$\ell_{0.5}$-regularizer}
         
     \end{subfigure}\\
     \begin{subfigure}[b]{0.3\textwidth}
         \centering
         \includegraphics[width=\textwidth]{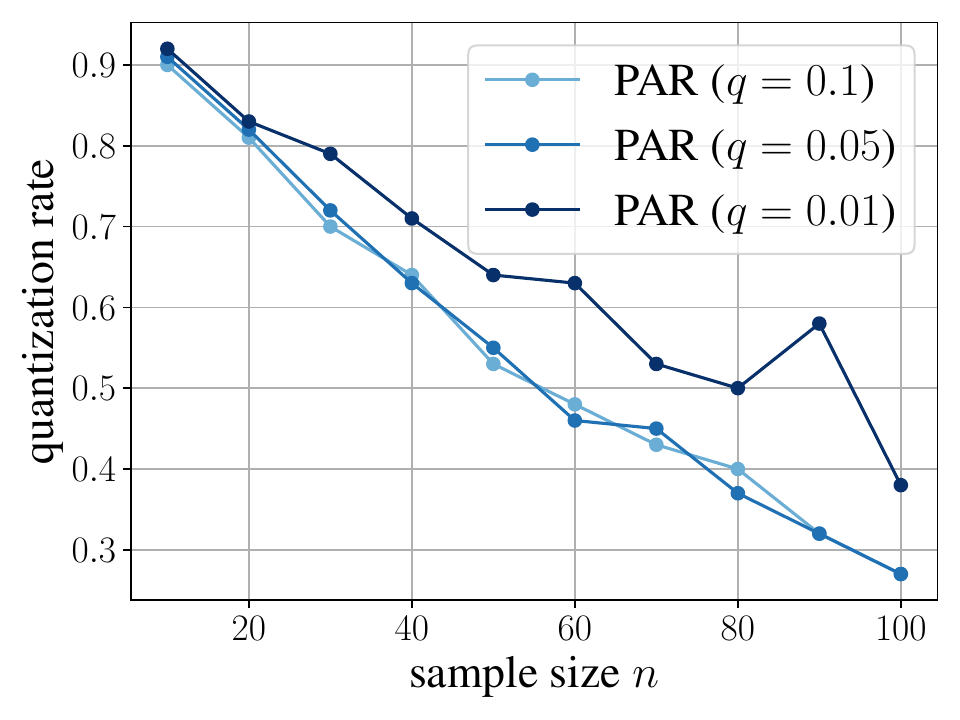}
         \caption{Ridge regularizer}
         
     \end{subfigure}
     \hfill
     \begin{subfigure}[b]{0.3\textwidth}
         \centering
         \includegraphics[width=\textwidth]{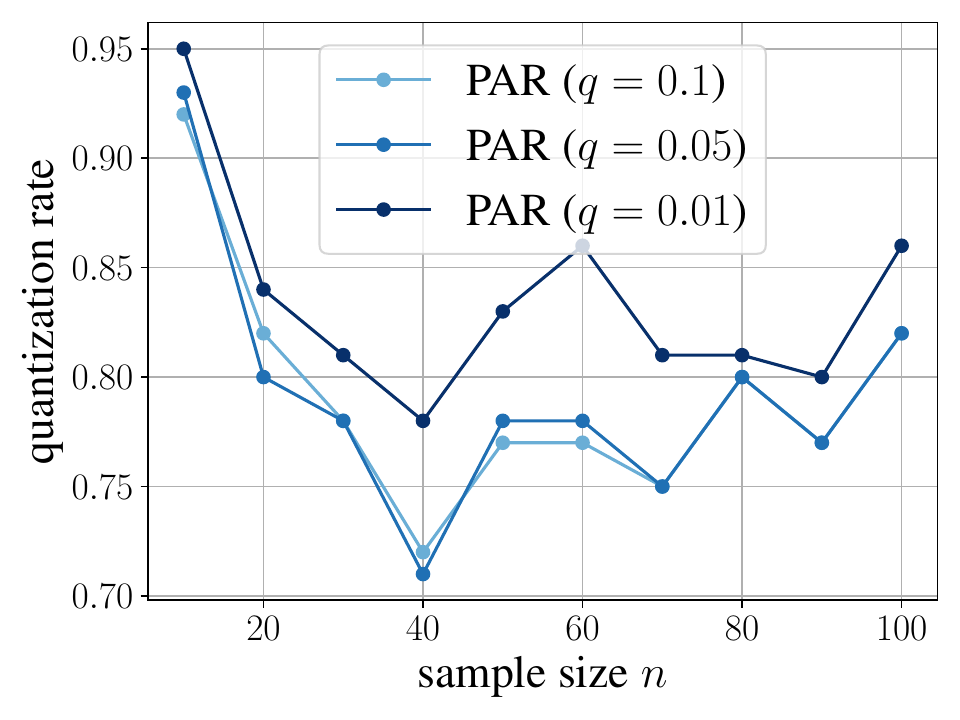}
         \caption{$\ell_1$-regularizer}
         
     \end{subfigure}
     \hfill
     \begin{subfigure}[b]{0.3\textwidth}
         \centering
         \includegraphics[width=\textwidth]{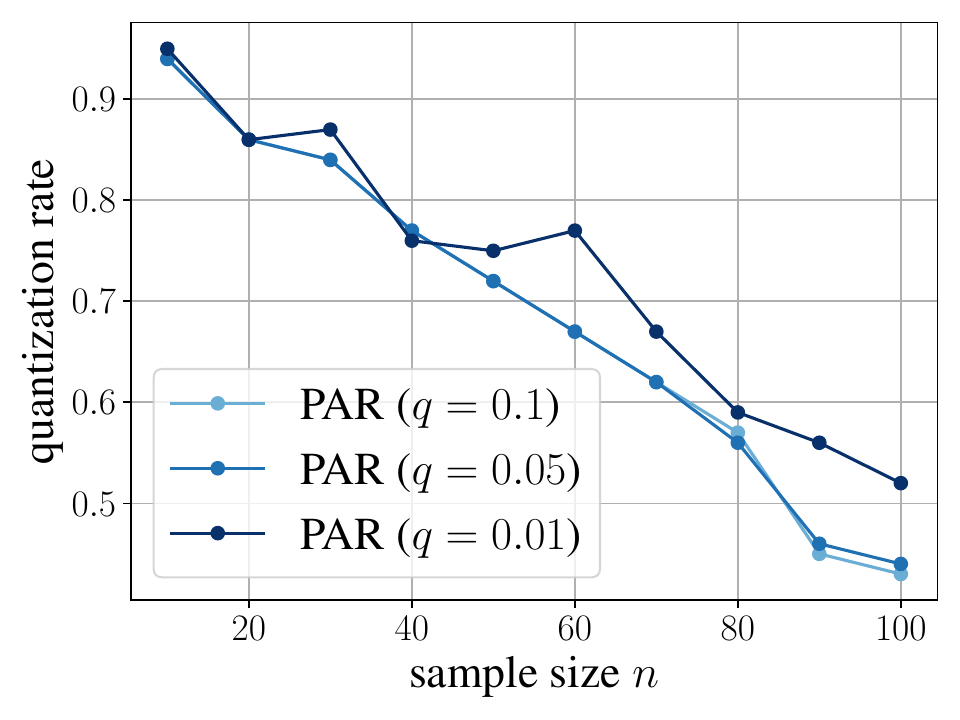}
         \caption{$\ell_{0.5}$-regularizer}
         
     \end{subfigure}
        \caption{Statistical performance of Ridge (left), $\ell_{1}$- (middle), and $\ell_{0.5}$-regularizers (right) and their PAR approximations on linear regression tasks. The quantization rate for each PAR is also shown.}
        \label{fig::statistical-performances-linear-regression}
\end{figure}

\begin{figure}[t]
     \centering
     \begin{subfigure}[b]{0.3\textwidth}
         \centering
         \includegraphics[width=\textwidth]{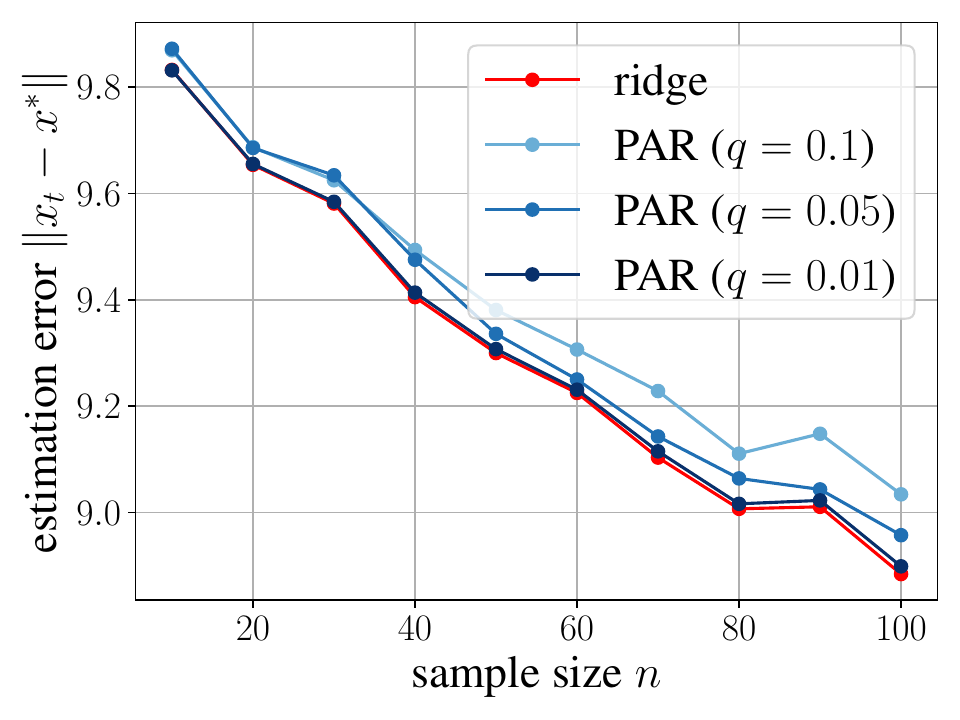}
         \caption{Ridge regularizer}
         \label{fig:estimation_error_ridge_logistic_regression}
     \end{subfigure}
     \hfill
     \begin{subfigure}[b]{0.3\textwidth}
         \centering
         \includegraphics[width=\textwidth]{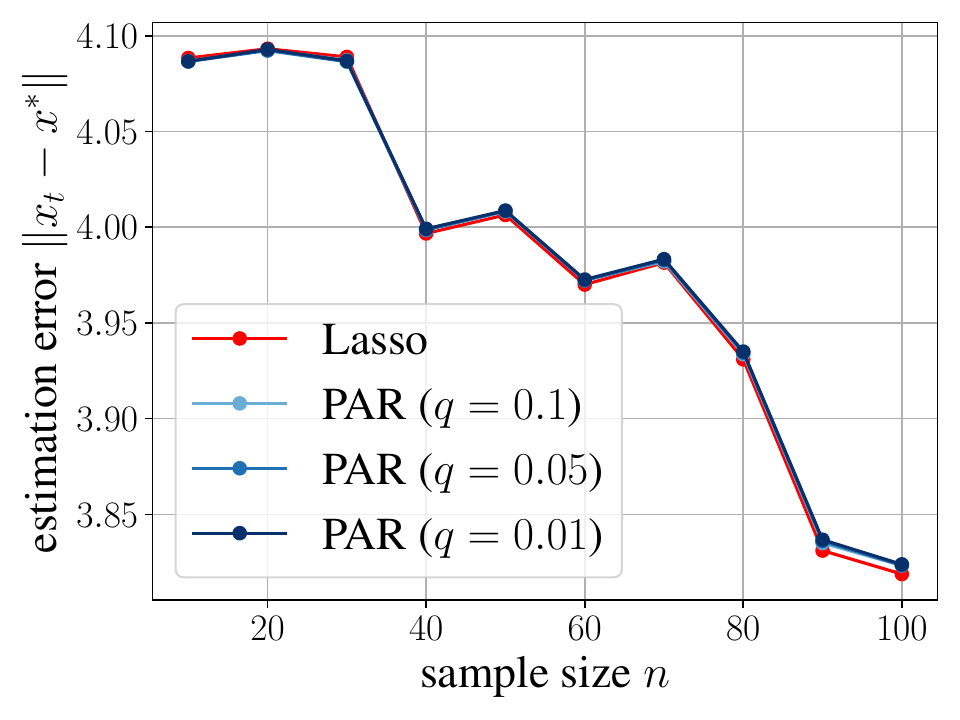}
         \caption{$\ell_1$-regularizer}
         
     \end{subfigure}
     \hfill
     \begin{subfigure}[b]{0.3\textwidth}
         \centering
         \includegraphics[width=\textwidth]{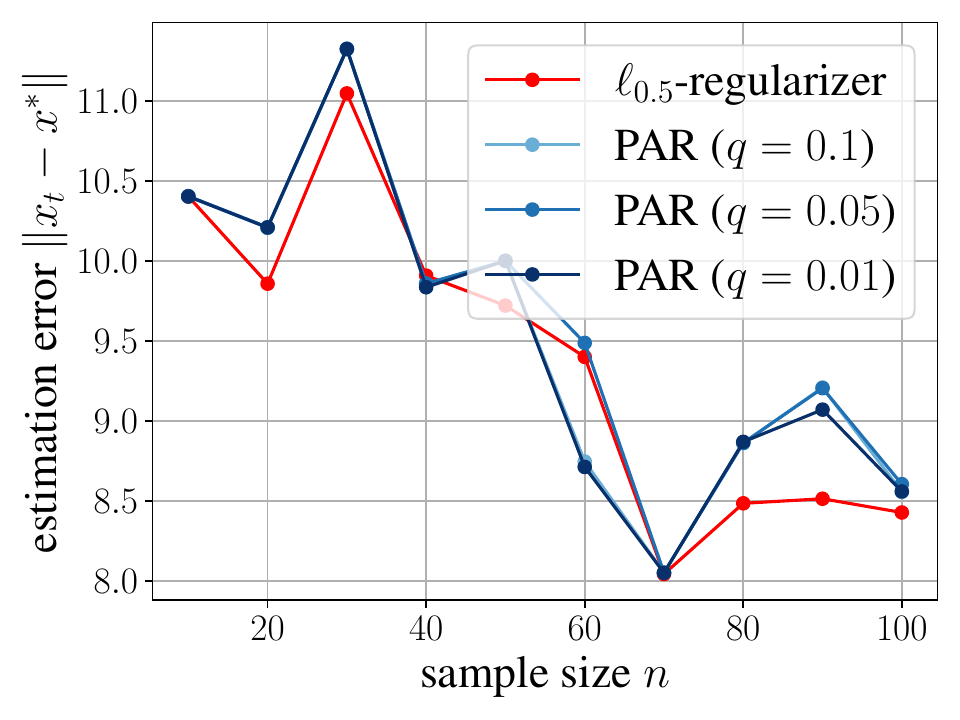}
         \caption{$\ell_{0.5}$-regularizer}
         
     \end{subfigure}\\
     \begin{subfigure}[b]{0.3\textwidth}
         \centering
         \includegraphics[width=\textwidth]{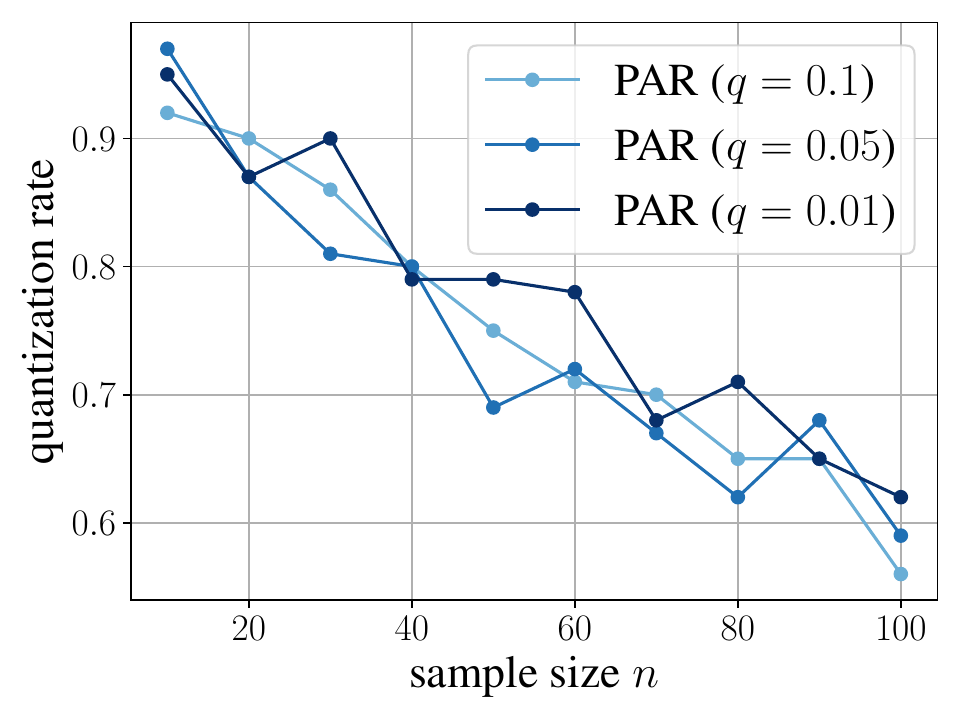}
         \caption{Ridge regularizer}
         
     \end{subfigure}
     \hfill
     \begin{subfigure}[b]{0.3\textwidth}
         \centering
         \includegraphics[width=\textwidth]{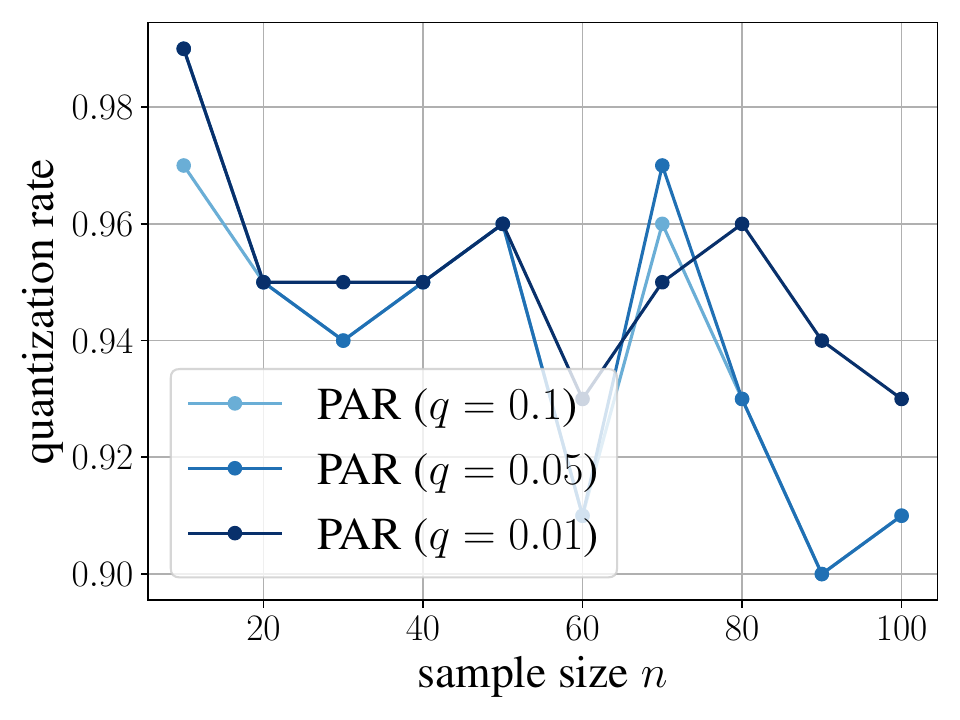}
         \caption{$\ell_1$-regularizer}
         
     \end{subfigure}
     \hfill
     \begin{subfigure}[b]{0.3\textwidth}
         \centering
         \includegraphics[width=\textwidth]{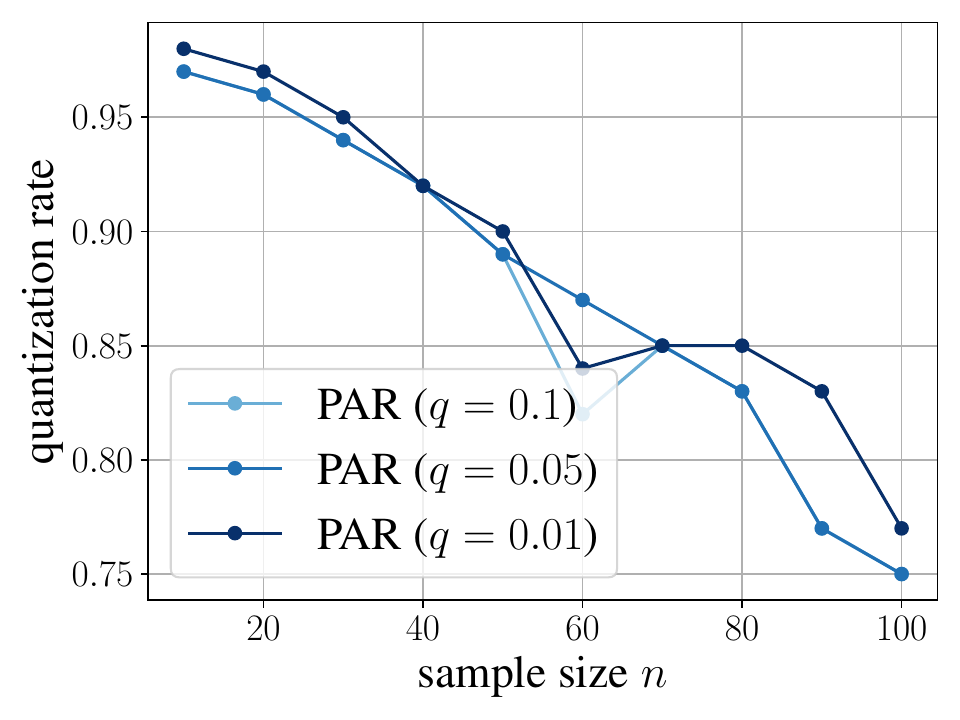}
         \caption{$\ell_{0.5}$-regularizer}
         
     \end{subfigure}
        \caption{Statistical performance of Ridge (left), $\ell_{1}$- (middle), and $\ell_{0.5}$-regularizers (right) and their PAR approximations on logistic regression tasks. The quantization rate for each PAR is also shown.}
        \label{fig::statistical-performances-logistic-regression}
\end{figure}

\subsection{Statistical guarantees}
\label{sec::exp-stats}
We assess the statistical accuracy of three classical regularizers: $\ell_2$-, $\ell_{1}$-, and $\ell_{0.5}$-regularizers, against their PAR approximations.  For each method, we measure the Euclidean distance $\norm{\hat{\vx}-\vx^{\star}}$ between the estimated parameter $\hat{\vx}$ and the true parameter $\vx^{\star}$.  All experiments are conducted with data dimension $d=200$, where the entries of the design matrix are drawn i.i.d. from $\cN(0,1)$. Responses are generated from either a linear or logistic model, with additive Gaussian noise of standard deviation $\sigma = 0.1$. 

For the PAR approximations, we consider three quantization gaps, $q \in \{0.1,0.05,0.01\}$,
and construct the corresponding PARs as described in \Cref{fig::par-as-different-regularizers}. All methods are evaluated on the same simulated datasets to ensure a fair comparison.

\Cref{fig::statistical-performances-linear-regression,fig::statistical-performances-logistic-regression} report the parameter estimation error $\norm{\hat{\vx}-\vx^{\star}}$ for each regularizer and its PAR counterpart on linear and logistic regression tasks, respectively. In both settings, PAR approximations achieve nearly identical estimation accuracy to their original counterparts across all quantization gaps. Combined with the observed quantization rates, these results confirm that PARs retain statistical accuracy while enabling quantization.

\section{Conclusion and future directions}
\label{sec:conclusion}
This work introduces Piecewise-Affine Regularized Optimization (PARO), a principled and versatile framework for inducing quantization while preserving optimization and statistical guarantees. For generalized linear models, and more broadly, supervised learning models, we prove that under mild design assumptions, every critical point of the PARO objective is at least \((1-n/d)\)-quantized. Here $n$ is the sample size and~$d$ is the parameter dimension, implying highly quantized solutions in the overparameterized regime where $d\gg n$. We also derive closed-form proximal mappings for three main PAR families: convex, quasiconvex, and nonconvex, and analyze the convergence of the proximal gradient method in the nonconvex setting. In the context of linear regression, we demonstrate that properly designed PARs can mimic the behavior of Ridge, Lasso, and general nonconvex penalties. They achieve comparable estimation and prediction performance while significantly reducing model storage. Extensive simulations validate our theoretical findings.

Below we point out a few interesting future directions.
\begin{itemize}
    \item \textbf{Learnable quantization values.}  
        Throughout this paper, we assume a fixed quantization set $\cQ$. Some prior works have shown that jointly learning the quantization set can substantially improve model performance \cite{esser2019learned,pouransari2020lsb}. Such approaches can also be interpreted through the lens of nonconvex piecewise-affine regularizers (PARs) \cite{yin2018binaryrelax}. An interesting direction is to investigate the quantization guarantees and convergence behavior under this learnable setting.

\item \textbf{Stochastic gradient methods for PARO.} In this work, we focused on deterministic (full batch) proximal gradient methods and ADMM. However, in large-scale machine learning settings with high-dimensional data and massive datasets, stochastic gradient methods become necessary for scalability. Unfortunately, the standard proximal stochastic gradient method does not induce the desired regularization effect (manifold identification) \cite{xiao2010rda}. This limitation has motivated the development of alternative stochastic optimization methods with strong manifold identification properties \cite{xiao2010rda,dockhorn2021demystifying,jin2025parq,qiu2025normalmapbasedproximalstochastic}. Further investigation in this direction, by exploiting the particular structure of PARs, can be very impactful in practice.

    \item \textbf{Applications to combinatorial optimization.}  
PARs hold promise for broader applications in combinatorial optimization involving discrete variables. In integer programming, linear relaxation, where binary constraints $x \in \{0,1\}$ are replaced with $x \in [0,1]$, combined with (possibly stochastic) rounding, has proven both theoretically and empirically effective. This relaxation corresponds to a convex PAR defined as $\Psi(x) = 0$ for $x \in [0,1]$ and $\Psi(x) = +\infty$ otherwise. Motivated by this, we conjecture that general PARs may induce meaningful discrete solutions for more complex combinatorial problems, including those with multi-level variables such as $x \in \{0,1,\dots,K\}$ for $K\geq 2$.
\end{itemize}

\section*{Acknowledgment}
We thank Dmitriy Drusvyatskiy for guidance on the convergence analysis of proximal gradient methods with nonconvex loss and regularization functions. We are also grateful to Lisa Jin for insightful discussions on deep learning model quantization.

\bibliography{ref}

\appendix

\section{Proofs for Proximal Mappings}
In this section, we present formal derivations of the explicit proximal mappings for the various PARs introduced in \Cref{sec::proximal-mappings}.
\subsection{Convex PAR}
Recall that the proximal mapping of $\Psi$ is defined as
\begin{equation}
\prox_{\lambda\Psi}(x) = \argmin_{z  \in \bR} \left\{\frac{1}{2}(z-x)^2 + \lambda\Psi(z) \right\}.
\end{equation}
Since $\Psi(z)$ is an even function, i.e., $\Psi(z) = \Psi(-z)$, its proximal operator is an odd function. We can therefore derive the solution for $x \ge 0$, which implies the minimizer $z$ is also non-negative, and then extend the result to all $x \in \bR$ using the relation $\prox_{\lambda\Psi}(x) = \sign(x)\prox_{\lambda\Psi}(|x|)$. For $x \ge 0$, the problem becomes
\begin{equation}
\prox_{\lambda\Psi}(x) = \argmin_{z  \ge 0} \left\{ \frac{1}{2}(z-x)^2 + \lambda\Psi(z) \right\}.
\end{equation}
The first-order optimality condition for this convex optimization problem asserts that the minimizer $z$ must satisfy
\begin{equation}\label{eq:optimality}
0 \in z-x + \lambda\partial\Psi(z),
\end{equation}
which can be rewritten as $x - z \in \lambda\partial\Psi(z)$. The subdifferential $\partial\Psi(z)$ for $z\geq 0$ is provided by 
\begin{equation}
\partial\Psi(z) =
\begin{cases}
[-a_0, a_0] & \text{if } z  = 0, \\
\{a_k\} & \text{if } z  \in (q_k, q_{k+1}) \text{ for } k \in \{0, \dots, m-1\}, \\
[a_{k-1}, a_k] & \text{if } z  = q_k \text{ for } k \in \{1, \dots, m\}.
\end{cases}
\end{equation}
We analyze the optimality condition (\eqref{eq:optimality}) for different cases.

\paragraph{Case 1: Solution is zero ($z=0$).}
The optimality condition is $x - 0 \in \lambda\partial\Psi(0)$, which means $x \in \lambda[-a_0, a_0]$. Since we consider $x \ge 0$, this implies $0 \le x \le \lambda a_0$. Thus, for $-\lambda a_0\leq x \le \lambda a_0$, the minimizer is $z=0$.

\paragraph{Case 2: Solution is a non-zero quantization point ($z = q_k$ for $k \in \{1, \dots, m\}$).}
The optimality condition is $x - q_k \in \lambda\partial\Psi(q_k)$, which translates to $x - q_k \in \lambda[a_{k-1}, a_k]$. This is equivalent to
\[
q_k + \lambda a_{k-1} \le x \le q_k + \lambda a_k.
\]
Thus, if $|x|$ lies in $[q_k+\lambda a_{k-1}, q_k+\lambda a_k]$, the solution is $\prox_{\lambda\Psi}(x) = \sign(x)q_k$.

\paragraph{Case 3: Solution lies strictly between quantization points ($z  \in (q_k, q_{k+1})$ for $k \in \{0, \dots, m-1\}$).}
Here, the subgradient is single-valued, $\partial\Psi(z) = \{a_k\}$. The optimality condition reduces to
\[
x - z = \lambda a_k,
\]
which gives the solution $z = x - \lambda a_k$. For this solution to be valid, it must lie in the assumed interval, $q_k < z < q_{k+1}$, which implies
\[
q_k < x - \lambda a_k < q_{k+1}.
\]
Rearranging for $x$ gives 
$q_k + \lambda a_k < x < q_{k+1} + \lambda a_k$.
Thus, if $|x|$ lies in $(q_k+\lambda a_k, q_{k+1}+\lambda a_k)$, the solution is $\prox_{\lambda\Psi}(x) = x - \sign(x)\lambda a_k$.

Combining the three cases above completes the proof.

\subsection{Quasiconvex PAR}
Recall that the proximal mapping is given by
\begin{equation}
    \prox_{\lambda\Psi}(x) = \argmin_{z \in \bR} \left\{\Phi(z):=\frac{1}{2}(z-x)^2 + \lambda\Psi(z) \right\}.
    \label{eq::82}
\end{equation}
Since $\Psi(x)$ is an even function, we can solve for $x \ge 0$ (which implies $z  \ge 0$) and generalize using $\prox_{\lambda\Psi}(x) = \sign(x)\prox_{\lambda\Psi}(|x|)$.

The first-order optimality condition for this problem is $x \in z + \lambda \partial\Psi(z)$. The Fréchet subdifferential $\partial\Psi(z)$ where $z\geq 0$ is given by\footnote{We slightly abuse notation here by using the same symbol as for the Clarke subdifferential.}
\begin{equation}
    \partial\Psi(z) =
\begin{cases}
[-1, 1] & \text{if } z  = 0, \\
\{1\} & \text{if } z  \in \left(kq, \frac{2k+1}{2}q\right), \\
\emptyset & \text{if } z  = \frac{2k+1}{2}q,\\
\{0\} & \text{if } z  \in \left(\frac{2k+1}{2}q, (k+1)q\right),\\
[0, 1] & \text{if } z  =kq \text{ where }k\neq 0.
\end{cases}
\end{equation}
In what follows, we consider two situations: $\lambda\leq q$ and $\lambda>q$.
\subsubsection{Case I: $\lambda\leq q$}
We consider the intersection between the two lines $y=x$ and $y=z+\lambda\partial\Psi(z)$. We further divide it into two cases.
\paragraph{Case 1: $kq\leq x\leq kq+\lambda$.} In this case, there is only one critical point $z=kq$. Hence, we have $\prox_{\lambda\Psi}(x)=kq$.
\paragraph{Case 2: $kq+\lambda\leq x\leq (k+1)q$.} In this case, there are two critical points: $z_1 = x - \lambda$ and $z_2 = x$. The global minimizer must be one of these candidates.

We proceed by comparing the objective function values
\begin{equation}\label{eq:key_inequality}
\Delta:=\frac{1}{2}(x-x)^2+ \Psi(x) - \left(\frac{1}{2}(x-\lambda-x)^2+\Psi(x-\lambda)\right) =\Psi(x)-\Psi(x-\lambda) - \frac{1}{2}\lambda^2.
\end{equation}
\begin{itemize}
    \item When $x\leq \frac{2k+1}{2}q$, we have 
\begin{equation}
    \Delta=\lambda^2 - \frac{1}{2}\lambda^2=\frac{1}{2}\lambda^2>0.
\end{equation}
Hence, $\prox_{\lambda\Psi}(x)=x-\lambda$.
\item When $\frac{2k+1}{2}q\leq x\leq \frac{2k+1}{2}q+\lambda$, we have 
\begin{equation}
    \Delta=\lambda \left(\frac{k+1}{2}q-\left(x-\lambda-\frac{k}{2}q\right)\right)-\frac{1}{2}\lambda^2.
\end{equation}
When $x\leq \frac{2k+1}{2}q+\frac{1}{2}\lambda$, we have $\Delta<0$, which implies that $\prox_{\lambda\Psi}(x)=x-\lambda$. When $\frac{2k+1}{2}q+\frac{1}{2}\leq x\leq \frac{2k+1}{2}q+\lambda$, we have $\Delta\geq 0$, which implies $\prox_{\lambda\Psi}(x)=x$. 
\item When $\frac{2k+1}{2}q+\lambda\leq x\leq (k+1)q$, we always have
$\Delta=-\frac{1}{2}\lambda^2<0$. Hence, $\prox_{\lambda\Psi}(x)=x$.
\end{itemize}

\subsubsection{Case II: $\lambda > q$}
We analyze the minimizers of $\Phi(z)$ in \eqref{eq::82}.

First, note that when $z \in \left[\frac{(2k+1)q}{2}, (k+1)q\right]$, $\Psi(z)$ is constant and $\Phi(z)$ is a quadratic function with minimum at $z = x$. Thus, if $x \in \left[\frac{(2k^\star+1)q}{2}, (k^\star+1)q\right]$ for some $k^\star\in \bZ_+$, then $x$ is a candidate minimizer. In other cases, since proximal mapping is nonexpensive and $\Phi(z)$ is a decreasing function in the range $[0, x]$, the other candidates are $\{(k+1)q\}_{k<k^\star}$.

When $z\in \left[kq, \frac{(2k+1)q}{2}\right]$, $\Psi(z)$ is affine and $\Phi(z)=\frac{1}{2}(z-x+\lambda)^2+C$ where $C$ is a constant. This quadratic achieves its minimum at $z = x - \lambda < kq$, which lies outside the interval, so the only candidates here are again grid points $\{kq\}_{k < k^\star + 1}$. 

Hence, all candidate minimizers belong to the set $\{(k+1)q\}_{k<k^\star+1}\cup \{x\}$. To identify the minimum, we first compare values at the grid points. For any $k$, we have
\begin{equation}
    \Phi(kq) = \frac{1}{2}(x-kq)^2+\frac{\lambda kq}{2}=\frac{q^2}{2}k^2+\frac{\lambda q-2xq}{2}k+\frac{1}{2}x^2.
\end{equation}
This quadratic in $k$ is minimized when $k=\left\lfloor \frac{x-\frac{1}{2}\lambda}{q}\right\rceil$. Next, we compare it with the candidate $x$. Note that
\begin{equation}
    \begin{aligned}
        \Phi(x)-\Phi\left(\left\lfloor \frac{x-\frac{1}{2}\lambda}{q}\right\rceil q\right)&=\frac{\lambda}{2}\cdot \left((k^\star+1)q-\left\lfloor \frac{x-\frac{1}{2}\lambda}{q}\right\rceil q\right) - \frac{1}{2}\left(x-\left\lfloor \frac{x-\frac{1}{2}\lambda}{q}\right\rceil q\right)^2\\
        &\geq \frac{1}{2}\left(x-\left\lfloor \frac{x-\frac{1}{2}\lambda}{q}\right\rceil q\right)\left(\lambda-\left(x-\left\lfloor \frac{x-\frac{1}{2}\lambda}{q}\right\rceil q\right)\right)\\
        &\geq 0.
    \end{aligned}
\end{equation}
where the last inequality uses $\lambda > q$. This confirms that the unique minimizer is $\left\lfloor \frac{x-\frac{1}{2}\lambda}{q}\right\rceil q$.
This completes the proof.
\subsection{Nonconvex PAR}
First, we consider that $x\in \left[q_k, \frac{q_k+q_{k+1}}{2}\right]$ for some $k$. It is obvious that $\Psi_{\lambda\Psi}(x)\in \left[q_k, \frac{q_k+q_{k+1}}{2}\right]$ for any $\lambda\geq 0$. Therefore, $\Psi_{\lambda\Psi}(x)$ is indeed the minimizer of a quadratic function:
    \begin{equation}
        \prox_{\lambda\Psi}(x)=\argmin_{z\in \left[q_k, \frac{q_k+q_{k+1}}{2}\right]}\left\{\frac{1}{2}\left(z+\lambda-x\right)^2\right\}=\clip\left(x-\lambda, q_k, \frac{q_k+q_{k+1}}{2} \right).
    \end{equation}
    Similarly, if $x\in \left[\frac{q_k+q_{k+1}}{2}, q_{k+1}\right]$ for some $k$, its proximal mapping can be derived by
    \begin{equation}
        \prox_{\lambda\Psi}(x)=\argmin_{z\in \left[\frac{q_k+q_{k+1}}{2}, q_{k+1}\right]}\left\{\frac{1}{2}\left(z-\lambda-x\right)^2\right\}=\clip\left(x+\lambda, \frac{q_k+q_{k+1}}{2},  q_{k+1}\right).
    \end{equation}
    This completes the proof.
\end{document}